\documentclass[11pt]{article} 
\usepackage[utf8x]{inputenc}

\RequirePackage{amsthm,amsmath,amsfonts,amssymb}
\RequirePackage{graphicx}
\RequirePackage{natbib} 
\usepackage{authblk} 
\usepackage[dvipsnames]{xcolor}
\definecolor{Plum}{rgb}{.4,0,.4} 
\definecolor{BrickRed}{rgb}{0.6,0,0} 
\usepackage[plainpages=false,pdfpagelabels,colorlinks=true,
linkcolor=RedOrange,urlcolor=RedOrange,citecolor=MidnightBlue]{hyperref}
\providecommand{\keywords}[1]{\small \textbf{\textit{Keywords---}} #1}

\usepackage{marginnote} 

\newcommand\numberthis{\addtocounter{equation}{1}\tag{\theequation}}
\renewcommand\qed{$\blacksquare$}
\usepackage{booktabs}
\usepackage{fullpage}
\usepackage{subcaption}
\usepackage[ruled]{algorithm2e}

\usepackage{bbm}
\usepackage{listings}
\def\ddefloop#1{\ifx\ddefloop#1\else\ddef{#1}\expandafter\ddefloop\fi}
\def\ddef#1{\expandafter\def\csname b#1\endcsname{\ensuremath{\mathbf{#1}}}}
\ddefloop ABCDEFGHIJKLMNOPQRSTUVWXYZ\ddefloop
\def\ddef#1{\expandafter\def\csname s#1\endcsname{\ensuremath{\mathsf{#1}}}}
\ddefloop ABCDEFGHIJKLMNOPQRSTUVWXYZ\ddefloop
\def\ddef#1{\expandafter\def\csname b#1\endcsname{\ensuremath{\mathbf{#1}}}}
\ddefloop abcdefghijklmnopqrstuvwxyz\ddefloop

\def\E{\mathop{\mathbb{E}}}
\def\Pr{\mathop{\mathbb{P}}}

\def\tr{\operatorname{Tr}}
\def\Reals{\mathbb{R}}
\def\Naturals{\mathbb{N}}

\def\Hil{\mathcal{H}}

\newtheorem{theorem}{Theorem}
\newtheorem{lemma}[theorem]{Lemma}
\newtheorem{proposition}[theorem]{Proposition}

\usepackage{xcolor}
\definecolor{codegreen}{rgb}{0,0.6,0}
\definecolor{codegray}{rgb}{0.5,0.5,0.5}
\definecolor{codepurple}{rgb}{0.58,0,0.82}
\definecolor{backcolour}{rgb}{0.95,0.95,0.92}

\lstdefinestyle{mystyle}{
    backgroundcolor=\color{backcolour},   
    commentstyle=\color{codegreen},
    keywordstyle=\color{magenta},
    numberstyle=\tiny\color{codegray},
    stringstyle=\color{codepurple},
    basicstyle=\ttfamily\footnotesize,
    breakatwhitespace=false,         
    breaklines=true,                 
    captionpos=b,                    
    keepspaces=true,                 
    numbersep=5pt,                  
    showspaces=false,                
    showstringspaces=false,
    showtabs=false,                  
    tabsize=2
}

\begin{document}
	\title{Universal Prediction Band via Semi-Definite Programming}

	\author[1]{Tengyuan Liang
	\thanks{\href{mailto:tengyuan.liang@chicagobooth.edu}{tengyuan.liang@chicagobooth.edu}.}}
\affil[1]{University of Chicago, Booth School of Business}
\date{}
\maketitle

\begin{abstract}
We propose a computationally efficient method to construct nonparametric, heteroscedastic prediction bands for uncertainty quantification, with or without any user-specified predictive model. Our approach provides an alternative to the now-standard conformal prediction for uncertainty quantification, with novel theoretical insights and computational advantages. The data-adaptive prediction band is universally applicable with minimal distributional assumptions, has strong non-asymptotic coverage properties, and is easy to implement using standard convex programs. Our approach can be viewed as a novel variance interpolation with confidence and further leverages techniques from semi-definite programming and sum-of-squares optimization. Theoretical and numerical performances for the proposed approach for uncertainty quantification are analyzed.
\end{abstract}

\keywords{Uncertainty quantification, heteroscedasticity, nonparametric prediction band, semi-definite programming, sum-of-squares.}


\section{Introduction}

A plausible criticism from the statistics community to modern machine learning is the lack of rigorous uncertainty quantification, with perhaps the exception in conformal prediction \citep{vovk2005algorithmic, lei2018DistributionFreePredictive, romano2019ConformalizedQuantile}. Instead, the machine learning community would argue that conventional uncertainty quantification based on idealized distributional assumptions may be too restrictive for real data. Nevertheless, without a doubt, uncertainty quantification for predictive modeling is essential to statistics, learning theory, and econometrics.  This paper will resolve the above inference dilemma by introducing a new method with provable uncertainty quantification via semi-definite programming. This paper provides an alternative approach to the now-standard conformal prediction for uncertainty quantification, with novel theoretical insights and computational advantages. The proposed method learns a data-adaptive, heteroscedastic prediction band that is: (a) universally applicable without strong distributional assumptions, (b) with desirable theoretical coverage with or without any user-specified predictive model, and (c) easy to implement via standard convex programs (when used in conjunction with a wide range of positive-definite kernels).

Let $(x, y) \in \mathcal{X} \times \mathbb{R}$ be the covariates and response data pair drawn from an unknown probability distribution $\mathcal{P}$. There are plenty of regression or predictive models---denoted by $\mathsf{m}_0(x)$---that estimate $\mathsf{m}(x) := \E[\by|\bx = x]$ sufficiently well with finite data. However, to make downstream decisions reliable, a good prediction band quantifying the uncertainty in $|\by - \mathsf{m}_0(\bx)|$ with provable coverage is urgently needed. The prediction band is of particular relevance to complex machine learning models that construct $\mathsf{m}_0(x)$ in a less transparent way, such as deep neural networks and boosting machines. 
This paper makes progress in filling in such a gap: we estimate a nonparametric, heteroscedastic prediction band $\widehat{\mathsf{PI}}(x)$ that enjoys provable coverage with minimal data assumptions for any predictive model. Our approach can be viewed as a novel variance interpolation with confidence and leverages techniques from sum-of-squares relaxations for nonparametric variance estimation. On a non-technical level, this paper enriches the toolbox of applied researchers with a theoretically justified new methodology for uncertainty quantification and visualization, as in conformal prediction.

\subsection{Semi-definite Programs and Prediction Bands}
\label{sec:Intro-SDP-PB}

We introduce our procedure for constructing the predictive band in this section.
Let $K(\cdot, \cdot): \mathcal{X} \times \mathcal{X} \rightarrow \mathbb{R}$ be a continuous symmetric and positive-definite kernel function. Given $n$ data pairs $\{ (x_i, y_i) \}_{i=1}^n$ and the corresponding kernel matrix $\bK \in \mathbb{S}^{n \times n}$ with $\bK_{ij}=K(x_i, x_j)$, our prediction band is constructed based on the following semi-definite program (SDP)
\begin{align*}
\min_{\bB}\quad & \tr(\bK \bB ) \\
\text{s.t.}\quad & \langle \sK_i , \bB \sK_i \rangle \geq (y_i - \mathsf{m}_0(x_i))^2, ~ i=1,\ldots, n \numberthis \label{eqn:SDP-general} \\
		& \bB \succeq 0
\end{align*}
where the optimization variable $\bB \in \mathbb{S}^{n\times n}$ is a symmetric positive semi-definite (PSD) matrix, $\mathsf{m}_0(\cdot): \mathcal{X} \rightarrow \mathbb{R}$ is a given predictive model (user-specified), and $\sK_i \in \mathbb{R}^n$ denotes the $i$-th column of the kernel matrix $\bK$. Given the estimated $\widehat{\bB}$, the \textbf{prediction band}, $\widehat{\mathsf{PI}}(x)$ that maps each $x$ to an interval, can be constructed accordingly
\begin{align*}
	& \widehat{\mathsf{PI}}(x, \delta) := \left[~ \mathsf{m}_0(x) - \sqrt{(1+\delta) \cdot \widehat{\mathsf{v}}(x)} ~,~ \mathsf{m}_0(x) + \sqrt{(1+\delta)\cdot \widehat{\mathsf{v}}(x)}  ~\right], ~ \forall x \in \mathcal{X}\;,\\
	&~~\text{where} \quad \widehat{\mathsf{v}}(x) := \langle \sK_x, \widehat{\bB} \sK_x \rangle  \;, \numberthis \\
	&~~\text{and} \quad \sK_x := [ K(x, x_1),\ldots, K(x, x_n) ]^\top \in \mathbb{R}^n \;,
\end{align*}
with $\delta \in \mathbb{R}$ being a scalar quantifying confidence. $\delta$ can be later calibrated for exact coverage control, and may be set as $0$ if $n$ is large. Here $\widehat{\mathsf{v}}(x)$ estimates the variability in the ``deviations'' $e_i := y_i - \mathsf{m}_0(x_i)$.
A few remarks on such deviations are in place. 
\begin{itemize}
	\item First, $e_i$'s can be computed based on any user-specified predictive model $\mathsf{m}_0(x)$ that estimates the conditional mean $\mathsf{m}_0(x) \approx \E[\by|\bx = x]$, be it accurate or not.
	\item Second, in the absence of such a predictive model for the conditional mean, one can set $\mathsf{m}_0(x) \equiv 0$ and learn a conditional second-moment function to assess uncertainty.
	\item Last, as shown next in \eqref{eqn:SDP-simultaneous}, in practice, one can simultaneously learn the conditional mean and variance functions using a variant of the above SDP. Therefore, a pre-specified model $\mathsf{m}_0(x)$ is not required.
\end{itemize}

Let $K^{\mathsf{m}}$ and $K^{\mathsf{v}}$ specify two kernel functions, corresponding to the conditional mean and variance functions respectively. $\bK^{\mathsf{m}}, \bK^{\mathsf{v}} \in \mathbb{S}^{n \times n}$ denote empirical kernel matrices on finite data with size $n$. For any $\gamma\geq 0$, the following convex SDP program constructs the prediction band and the conditional mean function simultaneously
\begin{align*}
\min_{\alpha, \bB}\quad & \gamma \cdot \big\langle \alpha, \bK^{\mathsf{m}} \alpha \big\rangle +  \tr\big(\bK^{\mathsf{v}} \bB \big) \\
\text{s.t.}\quad & \big\langle \sK_i^{\mathsf{v}} , \bB \sK_i^{\mathsf{v}} \big\rangle \geq \big( y_i - \big\langle \sK_i^{\mathsf{m}}, \alpha \big\rangle \big)^2, ~ i=1,\ldots, n \numberthis \label{eqn:SDP-simultaneous} \\
		& \bB \succeq 0
\end{align*}
where the optimization variables are $\bB \in \mathbb{S}^{n\times n}$ and $\alpha \in \mathbb{R}^n$. Given the solution $\widehat{\bB}$ and $\widehat{\alpha}$, the $\widehat{\mathsf{PI}}(x)$ is constructed as
\begin{align*}
	&\widehat{\mathsf{PI}}(x, \delta) :=  \left[~ \widehat{\mathsf{m}}(x) - \sqrt{(1+\delta) \cdot \widehat{\mathsf{v}}(x)} ~,~ \widehat{\mathsf{m}}(x) + \sqrt{(1+\delta) \cdot \widehat{\mathsf{v}}(x)}  ~\right], ~ \forall x \in \mathcal{X}\;,\\
	&\text{where}~  \widehat{\mathsf{m}}(x) := \big\langle \sK_i^{\mathsf{m}}, \widehat{\alpha} \big\rangle ~\text{and}~  \widehat{\mathsf{v}}(x) := \big\langle \sK_x^{\mathsf{v}}, \widehat{\bB} \sK_x^{\mathsf{v}} \big\rangle\;.
\end{align*}

\subsection{A Numerical Illustration}

Before diving into the motivations behind the above SDPs (Sec.~\ref{sec:sos-mni}) and corresponding theory (Sec.~\ref{sec:Theory-Uncertainty-Quantification}), let us first visually illustrate the empirical performance of the constructed prediction bands on a toy numerical example. A complete simulation study comparing our methods and conformal predictions will be deferred to Section~\ref{sec:comparision-conformal}. Details of the data generating processes will be elaborated therein as well. The quick exercise is to showcase that convex programs \eqref{eqn:SDP-general} and (\ref{eqn:SDP-simultaneous}) are easy to implement using standard optimization toolkits (say, CVX \citep{cvx}), and construct flexible prediction bands with desired coverage properties. As a motivating example, we try out the SDP (\ref{eqn:SDP-simultaneous}), which simultaneously estimates the conditional mean and variance functions. A minimal 10-line Python implementation is provided in Listing.~\ref{code}.


\begin{figure*}[tbhp]
	\centering
	    \begin{subfigure}{.5\textwidth}
	        \centering
	        \includegraphics[width=\linewidth]{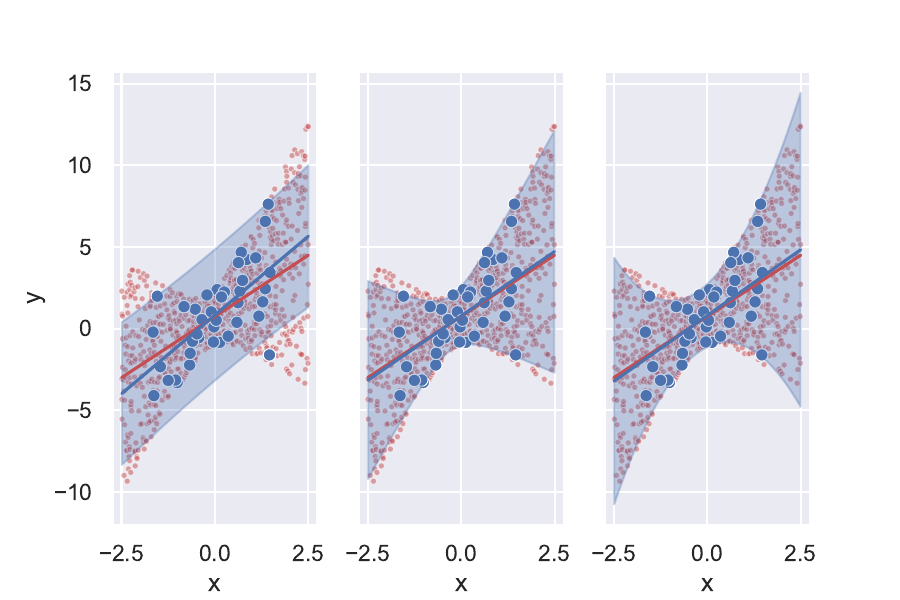}
			\caption{}
	        \label{fig:intro_example1}
	    \end{subfigure}%
	    \begin{subfigure}{0.5\textwidth}
	        \centering
	        \includegraphics[width=\linewidth]{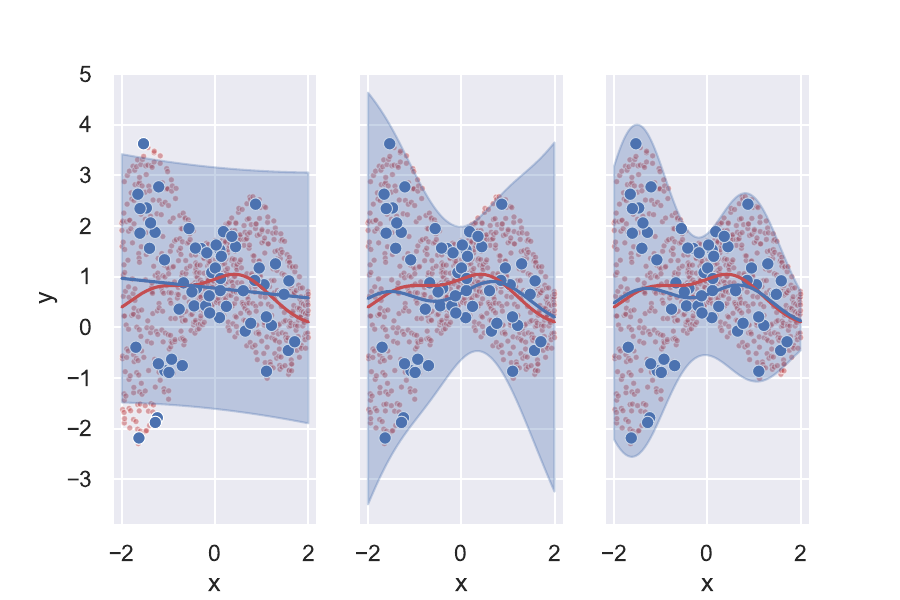}
			\caption{}
	        \label{fig:intro_example2}
	    \end{subfigure}
		\caption{\scriptsize From left to right: SLR, SDP1, and SDP2. For each plot, \textcolor{blue}{Blue dots} denote training data $\{(x_i, y_i)\}_{i=1}^n$, \textcolor{blue}{Blue line} denotes the estimated conditional mean $\widehat{\mathsf{m}}(x)$, and \textcolor{blue}{Blue band} denotes the estimated prediction band $\widehat{\mathsf{PI}}(x)$. \textcolor{red}{Red dots} represent the unknown test distribution, and \textcolor{red}{Red line} denotes the true conditional mean $\mathsf{m}(x) = \E[\by|\bx = x]$.  Here the training and test data share the same conditional distribution $\by|\bx = x$ and thus $\mathsf{m}(x)$. The training and test data are shared in three plots. A good coverage corresponds to when \textcolor{blue}{Blue band} covers essentially all \textcolor{red}{Red dots}. Statistics are summarized in Table~\ref{table:intro-example}.}
\end{figure*}


The first example is a linear model with heteroscedastic error: the conditional mean $\mathsf{m}(x)$ being a linear function and variance $\mathsf{v}(x)$ being a quadratic (with the conditional variance generated from a uniform distribution). We generate a training dataset of size $n = 40$ and compare the coverage among three methods: (a) SLR, simple linear regression, (b) SDP1, a SDP~(\ref{eqn:SDP-simultaneous}) with linear kernels $K^{\mathsf{m}}$ and $K^{\mathsf{v}}$ for both mean and variance functions, and (c) SDP2, a SDP~(\ref{eqn:SDP-simultaneous}) with a linear $K^{\mathsf{m}}$ and a (degree-$3$) polynomial kernel $K^{\mathsf{v}}$. The coverage is compared on the same test dataset of size $N=800$. Here $\gamma = 0.1$. 
See Fig.~\ref{fig:intro_example1} for details and Table.~\ref{table:intro-example} for coverage statistics. 

The second example is a non-linear, heteroscedastic error model: mean $\mathsf{m}(x)$ and variance $\mathsf{v}(x)$ functions lying in a reproducing kernel Hilbert space (RKHS) with a radial basis function (rbf) kernel. Here $n = 60$ training samples and $N = 800$ test samples are generated. Three methods being compared are: (a) SLR, (b) SDP1, rbf kernel for $K^{\mathsf{m}}$ and linear kernel for $K^{\mathsf{v}}$, and (c) SDP2, rbf kernels for both $K^{\mathsf{m}}$ and $K^{\mathsf{v}}$, summarized in Fig.~\ref{fig:intro_example2} and Table.~\ref{table:intro-example}. Here $\gamma = 1$.

\begin{table}
\centering
\caption{Simulated examples}
\label{table:intro-example}
\begin{tabular}{lrrrr}
	 & Coverage & Median Len & Average Len & MSE \\
\midrule
\multicolumn{4}{l}{Fig.~\ref{fig:intro_example1}: linear $\mathsf{m}(x)$, quadratic $\mathsf{v}(x)$} \\
\midrule
SLR  & 85.88\% & 8.2057 & 8.2658 & 0.6294 \\
SDP1 & 91.13\% & 7.4689 & \textbf{7.7173} & \textbf{0.1146}\\
SDP2 & \textbf{94.00\%} & \textbf{7.2962} & 8.3361 & 0.1720 \\
\midrule
\multicolumn{4}{l}{Fig.~\ref{fig:intro_example2}: rbf $\mathsf{m}(x)$, rbf $\mathsf{v}(x)$} \\
\midrule
SLR  & 96.13\% & 4.8048 & 4.8185 & 0.2556 \\
SDP1 & 99.25\% & 4.4138 & 4.6196 & 0.1916\\
SDP2 & \textbf{99.50\%} & \textbf{3.3488} & \textbf{3.7506} & \textbf{0.1670} \\
\bottomrule
\end{tabular}
\end{table}

%
%
%
%

%
%
%
%

These two numerical examples are minimal yet informative. In Fig.~\ref{fig:intro_example1}, SLR misspends a wide prediction bandwidth on data where the conditional variances are small yet fails to capture the large conditional variance cases, resulting in the overall coverage of $86\%$ and a median bandwidth of $8.21$. SDP1/SDP2 re-distributes the bandwidth budget leveraging the heteroscedastic nature and achieves an improved coverage $91\%$/$94\%$, with a smaller median bandwidth of $7.47$/$7.30$. Such an effect is even more pronounced in Fig.~\ref{fig:intro_example2}. Observe first that in SDP2, the prediction band constructed by \eqref{eqn:SDP-simultaneous} almost perfectly contours the heteroscedastic variances, thus achieving a $>99\%$ prediction coverage with a merely $3.35$ median bandwidth, in contrast to SLR with a $96\%$ coverage and a $4.80$ bandwidth. Second, a better conditional variance estimate also improves performance in learning the conditional mean, as seen in the differences between \textcolor{blue}{Blue lines} and \textcolor{red}{Red lines}. The errors are also numerically summarized in the column ``MSE'' of Table~\ref{table:intro-example}. Leveraging the heteroscedasticity in data, our prediction band distributes the bandwidth in a data-adaptive way, thus improving the overall coverage.

\subsection{Sum-of-Squares, Interpolation and Connections to Literature}
\label{sec:sos-mni}
The SDPs proposed in \eqref{eqn:SDP-general} and (\ref{eqn:SDP-simultaneous}) are inspired by recent advancements in optimization and learning theory. We will elaborate on the connections to related works and explain the innovations in our approach. 
We start with some basic observations about the SDPs. First, when $\alpha$ is not an optimization variable, (\ref{eqn:SDP-simultaneous}) recovers (\ref{eqn:SDP-general}). Second, the constraints set of (\ref{eqn:SDP-simultaneous}) is always non-empty since $\alpha = 0, ~\bB = \max_i  \|\sK_i^{\mathsf{v}} \|^{-2} y_i^2 \cdot \bI$ is feasible. 

\paragraph{Sum-of-Squares and Phase Retrieval}
As shown in Prop.~\ref{prop:representation}, the infinite-dimensional SDP with a nuclear-norm minimization is equivalent to \eqref{eqn:SDP-simultaneous}, 
	\begin{align*}
	\min_{\beta \in \Hil^{\mathsf{m}}, ~\bA: \Hil^{\mathsf{v}} \rightarrow \Hil^{\mathsf{v}}} \quad & \gamma \cdot \| \beta \|_{\Hil^{\mathsf{m}}}^2 +  \| \bA \|_{\star} \\
	\text{s.t.}\quad & \langle \phi_{x_i}^{\mathsf{v}} , \bA \phi_{x_i}^{\mathsf{v}} \rangle_{\Hil^{\mathsf{v}}} \geq \big( y_i - \langle  \phi_{x_i}^{\mathsf{m}}, \beta \rangle_{\Hil^{\mathsf{m}}} \big)^2, ~ \forall i \;.\\
			& \bA \succeq 0
	\end{align*}
Here $\Hil^{\mathsf{m}}, \Hil^{\mathsf{v}}$ denote two RKHSs where the conditional mean and variance functions reside. $\phi_{x_i} \in \Hil$ is the feature map w.r.t. the Hilbert space $\Hil$ and $\langle \cdot,\cdot \rangle_{\Hil}$ is the Hilbert space inner-product. We call it the infinite-dimensional SDP since the optimization variables $(\beta, \bA)$ are (function, operator) rather than finite-dimensional (vector, matrix). A few remarks are in place. First, if the kernel $K^{\mathsf{m}}$ is universal, $\langle  \phi_{x}^{\mathsf{m}}, \beta \rangle_{\Hil^{\mathsf{m}}}$ is dense in $L^2$ and hence can universally approximate all conditional mean function $\mathsf{m}(x)$. 
Second, as for the conditional variance which has positivity constraints over a continuum $x \in \mathcal{X}$ with
$0 \leq \mathsf{v}(x) = \big( y - \mathsf{m}(x) \big)^2$,  
we relax the positivity constraints using a sum-of-squares form
\begin{align}
	0 \leq \langle \phi_{x}^{\mathsf{v}} , \bA \phi_{x}^{\mathsf{v}} \rangle_{\Hil^{\mathsf{v}}} = \big( y - \mathsf{m}(x) \big)^2 ~, ~ \text{for some}~ \bA \succeq 0  \;. \label{eqn:sos-step}
\end{align}
It turns out that when $K^{\mathsf{v}}$ is universal, the above sum-of-squares function can approximate all smooth, positive functions \citep{fefferman1978positivity, bagnell2015learning, marteau-ferey2020NonparametricModels}, thus explaining the name ``universal'' in the title. Remark that sum-of-squares optimization \citep{lasserre2001global} for nonparametric estimation has recently been considered; see \citep{bagnell2015learning, marteau-ferey2020NonparametricModels, curmei2020shape}. The further relaxation changing from equality in (\ref{eqn:sos-step}) to inequality will be discussed in the next paragraph. Last, the minimum nuclear norm objective translates to a particular form of ``\textbf{minimum bandwidth}'' in the prediction band as $\mathsf{v}(x) = \langle \phi_x^{\mathsf{v}}, \bA \phi_x^{\mathsf{v}} \rangle_{\Hil^{\mathsf{v}}}$. In language, for all prediction bands that shelter the data, \eqref{eqn:SDP-simultaneous} aims to find the one with minimum bandwidth.

The curious reader may wonder where the nuclear norm $\| \bA \|_{\star} $ arises from. The first reason is conceptual: the nuclear norm is a relaxation for rank, and the procedure is to minimize the number of factors (rank) that explain the variance. The second reason is a connection to phase retrieval: specify $K^{\mathsf{m}}(x, x') \equiv 0$ and $K^{\mathsf{v}}(x, x') = \langle x, x' \rangle$ (the linear kernel with $\phi_{x}^{\mathsf{v}} = x$), and force the inequality constraints to be equal, our SDP in \eqref{eqn:SDP-simultaneous} is equivalent to phase retrievel
\begin{align*}
	\min_{\bA \succeq 0}~   \| \bA \|_{\star}, ~\text{s.t.}\quad & \langle x_i , \bA x_i \rangle = y_i^2, ~ \forall i = 1, \ldots, n \;.
\end{align*}
Conceptually, the minimum nuclear norm procedure estimates the smallest number of factors that could generate the variance.

\paragraph{Min-norm Variance Interpolation with Confidence}
Now we discuss the tuning parameter $\gamma \in [0, \infty]$ and reveal the connection to the recent min-norm interpolation literature \citep{liang2020JustInterpolate, bartlett2020BenignOverfitting, bartlett2021DeepLearning, ghorbani2020LinearizedTwolayers, montanari2020GeneralizationError, liang2021InterpolatingClassifiers}. In the limit of $\gamma \rightarrow 0$, \eqref{eqn:SDP-simultaneous} reduces to the familiar min-norm interpolation with kernel $\bK^{\mathsf{m}}$ (whenever it has full rank, since optimal $\bB=0$)
\begin{align*}
	\min_{\alpha}\quad & \big\langle \alpha, \bK^{\mathsf{m}} \alpha \big\rangle\\
	\text{s.t.}\quad & 0 = \big( y_i - \big\langle \sK_i^{\mathsf{m}}, \alpha \big\rangle \big)^2, ~ \forall i \;.
\end{align*}
In the limit of $\gamma \rightarrow \infty$, \eqref{eqn:SDP-simultaneous} reduces to  (since optimal $\alpha=0$)
\begin{align*}
	\min_{\bB}\quad & \tr(\bK^{\mathsf{v}} \bB ) \\
	\text{s.t.}\quad & \langle \sK^{\mathsf{v}}_i , \bB \sK^{\mathsf{v}}_i \rangle \geq y_i^2, ~ \forall i \;. \\
					& \bB \succeq 0 
\end{align*}
Now it is clear what the role of the \textit{tuning parameter} $\gamma$ is: it trades off the conditional mean $\mathsf{m}(x)$ and variance $\mathsf{v}(x)$ to explain the variability in $y$'s witnessed on the data. A small $\gamma$ aims to use a complex mean $\mathsf{m}(x)$ and a parsimonious variance $\mathsf{v}(x)$ to explain the overall variability, and vice versa. 

From the above discussion, it is also clear that the SDP~(\ref{eqn:SDP-simultaneous}) can be viewed as a min-norm \textbf{variance interpolation with confidence}. Instead of having the typical equality constraints in interpolation
$$
\big\langle \sK_i^{\mathsf{v}} , \bB \sK_i^{\mathsf{v}} \big\rangle = \big( y_i - \big\langle \sK_i^{\mathsf{m}}, \alpha \big\rangle \big)^2 \;,
$$
which violates the disciplined convex programming ruleset (due to the quadratic form on the RHS), we further relax to inequality constraints to incorporate additional ``confidence'' (and to make the problem convex at the same time)
$$
\big\langle \sK_i^{\mathsf{v}} , \bB \sK_i^{\mathsf{v}} \big\rangle \geq \big( y_i - \big\langle \sK_i^{\mathsf{m}}, \alpha \big\rangle \big)^2 \;.
$$
As we shall see, the notion of confidence in this variance interpolation is closely related to the notion of margin in classification \citep{bartlett1998BoostingMargin, liang2020PreciseHighdimensional}.

\paragraph{Support Vector Regression}
We illustrate that minor modifications to our SDP formulation lead to other problems, including support vector regression and kernel ridge regression. 
Specify the variance kernel as the trivial one $K^{\mathsf{v}}(x,x') = \mathbbm{1}(x=x')$, then the decision variable $\mathbf{B}$ only matters in its diagonal component, and our SDP \eqref{eqn:SDP-simultaneous} reduces to the kernel ridge regression
	\begin{align*}
		\min_{\alpha}\quad \gamma \cdot \langle \alpha , \mathbf{K}^{\mathsf{m}} \alpha \rangle + \sum_{i=1}^n \big( y_i - \big\langle \sK_i^{\mathsf{m}}, \alpha \big\rangle \big)^2 \;.
	\end{align*}
Moreover, a slight modification of \eqref{eqn:SDP-simultaneous} is to use the absolute deviation rather than the squared deviation in the constraints, namely
	\begin{align*}
		\big\langle \sK_i^{\mathsf{v}} , \bB \sK_i^{\mathsf{v}} \big\rangle \geq \big| y_i - \big\langle \sK_i^{\mathsf{m}}, \alpha \big\rangle \big| \;.
	\end{align*}
	In this case support vector regression is exactly our procedure with the specification $K^{\mathsf{v}}(x,x') = \mathbbm{1}(x=x')$, 
	\begin{align*}
		\min_{\alpha}\quad \gamma \cdot \langle \alpha , \mathbf{K}^{\mathsf{m}} \alpha \rangle + \sum_{i=1}^n \big| y_i - \big\langle \sK_i^{\mathsf{m}}, \alpha \big\rangle \big| \;.
	\end{align*}
	In summary, our SDP generalize beyond support vector machines, with the new non-trivial variance component for rigorous uncertainty quantification for heteroscedastic data.

\subsection{Literature Review}
\label{sec:Related-Work-Discussion}

There are increasingly many approaches proposed to address the uncertainty quantification dilemma in machine learning due to its significance and centrality. However, very few methods are theoretically grounded and universally applicable to the best of our knowledge. Many approaches are merely heuristics or data visualization tools. This section divides related theoretical studies in the literature into two categories and discusses how our method significantly differs from them and could potentially lead to a stronger theory.

\paragraph{Conformal Prediction}
Based on the exchangeability of data and a user-specified nonconformity measure, \cite{vovk2005algorithmic, shafer2008TutorialConformal} pioneered the field of conformal prediction, which uses past data to determine precise levels of confidence in new predictions. To some extent, the elegant theory of conformal prediction, motivated by online learning and sequential prediction, resolved the uncertainty quantification dilemma. The conformal prediction algorithm (see, for instance \cite{shafer2008TutorialConformal} Section 4.3) usually requires to enumerate over all the possibilities of $z = (x, y) \in \mathcal{X} \times \mathcal{Y}$, and for each possibility, calculate $n$ nonconformity measures via the leave-one-out method. Therefore, the total budget is $n \times |\mathcal{Y}| \times |\mathcal{X}|$, which can be expensive for continuous $y$ and multi-dimensional $x$. Much of the above computation can be saved if additional information about the metric structure in $x \in \mathcal{X}$ can be leveraged. In contrast, our SDP approach constructs the prediction band over all the $x$'s at once, leverages the metric structure in $\mathcal{X}$, and suffers at most a computational budget of $n^2$. An additional key feature in our approach is in the coverage theory established in Theorem~\ref{thm:coverage}: the prediction band has coverage probability $>95\%$ on a new data point $(\bx, \by)$, for $99.9999\%$ dataset $\{(x_i, y_i)\}_{i=1}^n$ of size $n$ drawn from the same distribution. Such a distinction on ``confidence'' vs. ``probability'' is discussed extensively in Section 2.2 of \cite{shafer2008TutorialConformal}.

There has been a vast line of recent work on extending the conformal prediction idea further to address the bottlenecks above in the regression setting. The body of work proliferates, and we certainly cannot do justice here. \cite{lei2018DistributionFreePredictive} alleviates the computational burden of the conformal prediction by introducing the sample-splitting technique. Remarkably, theory on the bandwidth is also studied in \cite{lei2018DistributionFreePredictive}, thus providing an angle to probe the statistical efficiency. \cite{romano2019ConformalizedQuantile} studies the problem that existing conformal methods can form nearly constant or weakly varying bandwidth and provide conservative results. \cite{romano2019ConformalizedQuantile} proposes conformalized quantile regression to address this issue. One shared feature of our SDP approach is that the prediction band is fully adaptive to heteroscedasticity. Finally, we would like to emphasize that conformal prediction constructs prediction bands in a numerical, black-box fashion without a structural understanding of the variance function. In contrast, our SDP approach provides a transparent and efficient way of learning the variance function, a complementary contribution to the conformal literature.

\paragraph{Residual Subsampling and Quantile Regression}
An alternative approach for uncertainty quantification that leverages the metric structure in $x \in \mathcal{X}$ is to resample the residuals locally. Typically, this is done by first fitting a predictive model $\mathsf{m}_0(x)$, and defining a local neighborhood around a new data $\bx$, then subsampling the residuals for uncertainty quantification via (conditional) quantiles. The validity of the above approach crucially depends on how many ``similar residuals'' to pool information from. However, the curse of dimensionality comes in since data points are far from each other in high dimensions, posing challenges in pooling the residuals.
One can also use either the obtained residuals or the original responses $y$ to fit a conditional quantile regression model \citep{koenker1978regression, koenker2001quantile, belloni2011, belloni2019conditional}, 
$\widehat{\xi}^{\tau}(\cdot) := \arg\min_{\xi} \tfrac{1}{n} \sum_{i=1}^n \rho_{\tau}\big(y_i - \xi(x_i) \big)$ where $\tau \in (0,1)$ is a quantile parameter, $\rho_{\tau}(\cdot):\Reals \rightarrow \Reals_+$ is the tilted absolute value function, and $\widehat{\xi}^{\tau}(\cdot): \mathcal{X} \rightarrow \Reals$ is the estimated conditional quantile function.
However, it is not guaranteed that over all $x \in \mathcal{X}$, the estimated conditional quantile function satisfies $\widehat{\xi}^{\tau_1}(x) < \widehat{\xi}^{\tau_2}(x)$ for two quantiles $\tau_1 < \tau_2$. In other words, it is entirely possible that for several $x$'s, the conditional prediction intervals are empty \citep{chernozhukov2010quantile}.

\section{Theory for Uncertainty Quantification}
\label{sec:Theory-Uncertainty-Quantification}

In this section, we develop a theory for the coverage property of the prediction band constructed above, under the mild assumption that the data are i.i.d. drawn with $(\bx, \by)\sim\mathcal{P}$. To highlight the main arguments in a simple form, let us consider the setting $\mathsf{m}_0(x) \equiv 0$. Otherwise, the same proof follows by replacing $\by$ with $\by - \mathsf{m}_0(\bx)$. Define the corresponding prediction band, with a confidence parameter $\delta \in (0,1]$
\begin{align}
	\widehat{\mathsf{PI}}(x, \delta) = \left[\pm   \sqrt{(1+\delta) \cdot \widehat{\mathsf{v}}(x)}   \right] \;.
\end{align}
We need the following assumptions before stating the theorem, where $(\bx, \by)\sim\mathcal{P}$ and $\text{C}>0$ denotes a universal constant.
\begin{itemize}
	\item[\textsf{[S1]}](Kernel and RKHS) The continuous symmetric kernel $K$ is positive definite and satisfies $\sup_{x \in \mathcal{X}} K(x, x) \leq \text{C}$. In addition, eigenvalues of the associated integral operator $\mathcal{T}:\Hil \rightarrow \Hil$ satisfy $\lambda_j(\mathcal{T}) \leq \text{C} j^{-\tau}, j\in \Naturals$ for some constant $\tau>1$.
	\item[\textsf{[S2]}](Non-trivial uncertainty) There exist constants $\eta \in (0,1), \xi>0$ such that $\Pr\big[ \by^2 > \xi \cdot K(\bx, \bx) ~ |~ \bx = x \big] > \eta$ holds for all $x \in \mathcal{X}$.
	\item[\textsf{[S3]}](Non-wild uncertainty) There exists a constant $\omega>0$ such that
	$\Pr\big[ \by^2 > t  \cdot K(\bx, \bx)\big] <  \exp(- \text{C} t^{\omega})$ for all $t\geq 1$.
\end{itemize}

\paragraph{Discussion of Assumptions}
All the above assumptions are mild. The eigenvalue decay in \textsf{[S1]} is almost identical to $\tr(\mathcal{T}) < \infty$ (bounded trace integral operator). \textsf{[S2]} is also minimal, since it is only not true when there is no variability in $\by|\bx = x$. \textsf{[S3]} is the most stringent one, which requires the variability of $\by$ to exhibit a certain tail-decay. For small $\omega \in (0,1)$, \textsf{[S3]} can be much milder than exponential tail-decay. Bounded or Gaussian $\by|\bx=x$ satisfies \textsf{[S3]} with arbitrarily large $\omega$ or $\omega = 2$, respectively. With some extra work, \textsf{[S3]} can be relaxed to the case of a sufficiently rapid polynomial tail-decay. \textsf{[S2]} can be relaxed to restricting only to $x$ with $\Pr\big[ \by^2 > 0 ~ |~ \bx = x \big] = 1$.

\begin{theorem}
	\label{thm:coverage}
	Define the objective value of the SDP in \eqref{eqn:SDP-general}
	\begin{align*}
		\widehat{\mathsf{Opt}}_n:= \min_{\bB}\quad & \tr(\bK \bB ) \\
		\text{s.t.}\quad & \langle \sK_i , \bB \sK_i \rangle \geq y_i^2, ~ i=1,\ldots, n \;.\\
				& \bB \succeq 0
	\end{align*}
	Assume that \textsf{[S1]}-\textsf{[S3]} hold.
	For any $\delta \in (0, 1]$, the following non-asymptotic, data-dependent prediction band coverage guarantee holds,
	\begin{align*}
		\Pr_{(\bx, \by)\sim \mathcal{P}}\big[ \by \in \widehat{\mathsf{PI}}(\bx, \delta) \big] &\geq 1 - \delta^{-1}  (\widehat{\mathsf{Opt}}_n \vee 1)   \sqrt{\tfrac{ \text{C}_{\tau, \xi, \eta, \omega} \cdot \log (n)}{n}} \;,\\
		\text{and}~~ \widehat{\mathsf{Opt}}_n &\leq  \big[ \log(n) \big]^{\text{c}_{\omega}} \;,
	\end{align*}
	with probability at least $1-n^{-10}$ on $\{(x_i, y_i)\}_{i=1}^n$. Here the constants $\text{C}_{\tau, \xi, \eta, \omega}, \text{c}_{\omega}$ only depend on parameters in \textsf{[S1]}-\textsf{[S3]}.
\end{theorem}

\subsection{What does the Theorem Entail}

A few remarks are in order before we sketch the proof of Theorem~\ref{thm:coverage}.

\paragraph{Coverage} First, the above theorem says that the prediction band constructed using the SDP based on a dataset of size $n$, will correctly cover a fresh data point $(\bx, \by) \sim \mathcal{P}$ drawn from the same distribution, with a non-asymptotic coverage probability (on the new data $\bx, \by$)
$$
1 -  \delta^{-1}  \sqrt{\tfrac{\log^3(n)}{n}} \;.
$$
With $\delta = 0.5$, the bandwidth $\text{Length}\big[ \widehat{\mathsf{PI}}(x) \big] = 2.45\sqrt{\widehat{\mathsf{v}}(x)}$ is at a heteroscedastic level adaptive to $x$ with corresponding coverage probability at least $1 - O^\star(\sqrt{\tfrac{1}{n}})$. Here $O^\star$ hides polylog factors. The coverage can be arbitrary close to $1$ with large $n$ without the need of increasing $\delta$, which is in clear distinction to the conventional wisdom that coverage $1$ can only be possible with an increasing $\delta$ regardless of $n$.
Again, we emphasize that the above coverage guarantee holds essentially on $99.9999\% \ll 1-n^{-10}$ of the datasets $\{(x_i, y_i\}_{i=1}^n$. In Section~\ref{sec:calibration}, we propose a rigorous calibration algorithm to choose $\delta^\star(\alpha)$ to achieve a constant coverage level $1-\alpha \in (0, 1)$.

\paragraph{Optimality} If one wishes to obtain the classic $95\%$ coverage probability, then choosing $\delta = O^\star(\sqrt{\tfrac{1}{n}})$ suffices, which translates to
\begin{align}
	\label{eqn:len-ours}
	\text{Length}\big[ \widehat{\mathsf{PI}}(x) \big] = \big(1+ O^\star(\sqrt{\tfrac{1}{n}})\big) \cdot \sqrt{\widehat{\mathsf{v}}(x)} \;.
\end{align}
Recall that in classic simple linear regression, the prediction interval is of length
\begin{align}
	\label{eqn:len-slr}
	\big( 1 + \sqrt{ \tfrac{1}{n} + \tfrac{(x - \bar{x})^2}{\Sigma_i (x_i - \bar{x})^2}} \big) \cdot 3.92 \hat{s}
\end{align}
with $\hat{s}=\sqrt{\tfrac{\Sigma_i \hat{e}_i^2}{n-2}}$ being the estimated residual standard error. The fact that \eqref{eqn:len-ours} and (\ref{eqn:len-slr}) share the $\sqrt{\tfrac{1}{n}}$ fluctuation seems to indicate the optimality of our Theorem (in terms of the dependence on $n$).

\paragraph{Data Adaptivity} Curiously, the objective value of the convex optimization program quantifies the uncertainty of the prediction band: a smaller $\widehat{\mathsf{Opt}}_n$ implies (a) a better confidence/coverage guarantee and (b) a narrower prediction band overall. More importantly, the $\widehat{\mathsf{Opt}}_n$ can be calculated directly from data! We find such an optimization/inference interface exciting: the data-adaptive bound lets us know the coverage guarantee specific to the current dataset. Put differently, the convex program constructs the prediction band via its solution and at the same time, reveals the confidence via its objective value. Since $\widehat{\mathsf{Opt}}_n$ is a function of the dataset, our Theorem reveals which dataset allows for a better prediction band. Remark that $\widehat{\mathsf{Opt}}_n = \| \widehat{\mathsf{v}} (\cdot) \|_{\star}^2$ is also a particular norm of the heteroscedastic variance function, quantified by the nuclear norm of the associated PSD operator $\widehat{\bA} \succeq 0$ with $\widehat{\mathsf{v}}(x) = \langle \phi_x, \bA \phi_x \rangle_{\Hil}$. Curiously, a simpler variance function $\widehat{\mathsf{v}} (x)$ (with a small norm) will simultaneously result in a narrower band and better coverage. We emphasize that the above discussion is in sharp contrast to the conventional wisdom that a narrow band usually leads to poor coverage guarantees.

\subsection{Calibration and Coverage Control}
\label{sec:calibration}
One nice feature about conformal prediction is that it directly operates on a user-specified coverage level (e.g., $95\%$), albeit the resulting procedure only achieves some form of coverage guarantee in the marginal sense. In contrast, the coverage level guarantee of the current SDP approach in Theorem~\ref{thm:coverage} is in an inequality form with a mild dependence on the non-explicit universal constant $C_{\gamma, \xi, \eta, \omega}$; however, the coverage is in a stronger conditional sense conditioned on $\{ (x_i, y_i) \}_{i=1}^n$. The constant won't significantly affect the coverage guarantee in the large $n$ setting; nevertheless, curious readers may wonder if more transparent control could be achieved by tuning $\delta$. This section provides a theoretically justified calibration procedure in choosing $\delta$ (in the SDP band) to control coverage at a user-specified level $1-\alpha \in (0, 1)$, as in conformal prediction. Such fine calibration on $\delta$ can be helpful numerically in the moderate $n$ and constant $\alpha$ setting (e.g., $5\%$).

The calibration idea is based on sample-splitting. Split the samples into two parts, the training set $\{ (x_i, y_i) \}_{i=1}^{n} $ and the calibration set $\{ x'_j, y'_j\}_{j=1}^m$, with in total $n+m$ data points drawn i.i.d. from $\mathcal{P}$. The training set will be used to construct prediction band $\widehat{\mathsf{PI}}(\cdot, \delta)$. The calibration data set will be used to choose $\delta$ to calibrate coverage. 

\paragraph{Calibration Procedure}
Suppose that there exsits large enough constants $\Delta, L>0$ such that
\begin{align*}
	\Pr_{(\bx, \by)\sim \mathcal{P}}\big[ \by \in  \widehat{\mathsf{PI}}(\bx, \Delta) \big] = 1\;, ~\left| \frac{d}{d \delta} \Pr_{(\bx, \by)\sim \mathcal{P}}\big[ \by \notin \widehat{\mathsf{PI}}(\bx, \delta) \big] \right| < L \;.
\end{align*}
The calibration uses a dyadic search to select $\delta^\star(\alpha) \in [-1, \Delta]$ with the set $\{ x'_j, y'_j\}_{j=1}^m$. The goal of the calibration is to ensure $\Pr_{(\bx, \by)\sim \mathcal{P}}\big[ \by \in \widehat{\mathsf{PI}}(\bx, \delta^\star(\alpha)) \big] \geq 1- \alpha$. 
\begin{algorithm}
\caption{Calibration for Coverage Control}\label{alg:calibration}
\KwData{Calibration set $\{ x'_j, y'_j\}_{j=1}^m$, coverage level $1-\alpha$\;}
\KwResult{Calibrated $\delta^\star(\alpha)$ \;}
Initialize $\delta = -1$\;
\While{$\tfrac{1}{m} \sum_{j=1}^m \mathbbm{1} [y_j' \notin \widehat{\mathsf{PI}}(x_j', \delta)]  > \tfrac{3}{4}\alpha$, }{
   $\delta \gets \tfrac{\delta + \Delta}{2}$ \;
}
\Return $\delta^\star(\alpha) \gets \delta$.
\end{algorithm}

The following result derives the theoretical ground for the calibration procedure. We defer the proof to Section~\ref{sec:calibration-proof}.
\begin{lemma}[Calibration]
	\label{lem:calibration}
	Consider the calibration procedure in Algorithm~\ref{alg:calibration} and $L, \Delta, \alpha$ specified therein. If the size of the calibration set $m$ is large enough such that $$\sqrt{\tfrac{\log \big( \lceil \log_2\big(\tfrac{2L(\Delta+1)}{\alpha}\big) \rceil + 1 \big) + 10\log (m)}{m}} \leq \tfrac{1}{4}\alpha\;,$$ then the calibrated $\delta^\star(\alpha)$ satisfies the coverage control
	\begin{align*}
		\Pr_{(\bx, \by)\sim \mathcal{P}}\big[ \by \in \widehat{\mathsf{PI}}(\bx, \delta^\star(\alpha)) \big] \geq 1- \alpha \;,
	\end{align*}
	with probability at least $1 - 2m^{-10}$ on the calibration set $\{ x'_j, y'_j\}_{j=1}^m$.
\end{lemma}

\subsection{Intuition and Proof Sketch}

We first explain the key intuition before presenting the details of the proof sketch. The proof first leverages a representation theorem that relates the finite-dimensional (kernelized) SDP to an infinite-dimensional SDP to decouple the dependencies among $x_i$'s. Next, we propose to use empirical process theory to analyze the prediction coverage, inspired by the margin-based analysis originally done in analyzing classification. Finally, in controlling the uniform deviations between empirical and population coverage, we use properties of the PSD operators and conditional quantile functions. The additional calibration procedure in fine-tuning the confidence parameter $\delta^\star(\alpha)$ for a fixed coverage level $1-\alpha \in (0, 1)$ also hinges on uniform deviation arguments.
Our analysis fundamentally relies on empirical process theory and is crucially different from conformal prediction analysis (based on exchangeability). Since the SDP provides a rigorous coverage guarantee as conformal prediction, we hope the new proof idea opens new doors to study uncertainty quantification.

Now we sketch the proof of Theorem~\ref{thm:coverage}. 
	Observe that by definition
	\begin{align*}
		\text{\small (LHS)}:= \Pr_{(\bx, \by)\sim \mathcal{P}}\big[ \by \notin \widehat{\mathsf{PI}}(\bx, \delta) \big] = \E_{(\bx, \by)\sim \mathcal{P}} \big[ \mathbbm{1}( \by^{-2} \widehat{\mathsf{v}}(\bx) <\tfrac{1}{1+\delta}  ) \big]\;.
	\end{align*}
	Define a hinge function $h_{\delta}(t): t \mapsto \max\{ \tfrac{1+\delta}{\delta}(1-t) , 0 \}$, we have
	\begin{align*}
		\mathbbm{1}(t < \tfrac{1}{1+\delta} ) \leq h_{\delta}(t), ~\forall t \in \Reals \;,
	\end{align*}
	and thus
	\begin{align*}
		\text{\small (LHS)} \leq \E_{(\bx, \by)\sim \mathcal{P}} \big[ h_{\delta}( \by^{-2} \widehat{\mathsf{v}}(\bx) ) \big] \;. \numberthis \label{eqn:lhs-start}
	\end{align*}

	Define a real positive function (indexed by $\bA$) on the data $z=(x, y)$,
		$
		f_{\bA}(z): z \mapsto \big\langle \tfrac{\phi_x}{y}, \bA \tfrac{\phi_x}{y} \big\rangle_{\Hil} \;.
		$
	Here $\tfrac{\phi_x}{y} \in \Hil$ lies in the RKHS, and $\bA: \Hil \rightarrow \Hil$ is a PSD operator. 
	Define a sequence of function spaces according to its nuclear-norm radius
	$
		\mathcal{F}_{k} := \{ f_{\bA} ~:~  2^{k-1}< \| \bA \|_\star \leq 2^k  \}
	$ for all $k \in \Naturals$ and $\mathcal{F}_{0} := \{ f_{\bA} ~:~   \| \bA \|_\star \leq 1  \}$.

	With the Prop.~\ref{prop:representation} establishing the equivalence between the kernelized SDP and the infinite-dimensional SDP, $\widehat{\bA} = \sum_{i,j} \widehat{\bB}_{ij} \phi_{x_i} \otimes \phi_{x_j}$,
	we know that $y^{-2} \widehat{\mathsf{v}}(x) = f_{\widehat{\bA}}(z)$. There exists a $k \in \Naturals$ such that $f_{\widehat{\bA}} \in \mathcal{F}_k$ with $ 2^{k-1} \leq \widehat{\mathsf{Opt}}_n < 2^k$,
	and thus we continue to bound
	\begin{align*}
		\text{\small (LHS)} &\leq \E_{(\bx, \by)\sim \mathcal{P}} \big[ h_{\delta}\circ f_{\widehat{\bA}} (\bz) \big] \\
		&\leq \underbrace{\widehat\E \big[ h_{\delta}\circ f_{\widehat{\bA}} (\bz) \big]}_{\text{(i)}} + \underbrace{\sup_{f \in \mathcal{F}_k} \left(\E - \widehat{\E}\right) [ h_{\delta} \circ f ]}_{\text{(ii)}} \;. 
	\end{align*}

	For term (i), recall the optimality condition of \eqref{eqn:SDP-general}, 
	\begin{align*}
		\langle \sK_i , \widehat{\bB} \sK_i \rangle \geq y_i^2  ~\Leftrightarrow~ f_{\widehat{\bA}}(z_i) \geq 1
	\end{align*}
	which further implies $h_{\delta} \circ f_{\widehat{\bA}}(z_i) = 0$ for all $i=1, \cdots, n$. Therefore term (i) is zero.

	For term (ii), we will use the high probability symmetrization in Prop.~\ref{prop:symmetrization}. Introduce i.i.d. Rademacher variables $\{ \epsilon_i\}_{1}^n$ independent of the data. Note that we only need to consider $k \leq k_0$ such that
	$2^{k_0} = [\log(n)]^{ \text{c}_{\omega}}$, where we use the upper estimate on $\widehat{\mathsf{Opt}}_n$ obtained in Prop.~\ref{prop:growth-opt}, which is implied by Assumption \textsf{[S3]}. 
	With probability at leat $1 - 2\exp(-t)$ on the data $\{ z_i\}_{1}^n$,  uniformly over all $k \leq k_0$
	\begin{align*}
		\text{(ii)} &\leq 2 \cdot \E_{\{\epsilon_i\}_{1}^n} \sup_{f \in \mathcal{F}_k}~ \frac{1}{n} \sum_{i=1}^n \epsilon_i  \big(h_{\delta} \circ f\big)(z_i) + \text{(iii)}\\
		& \leq 2\cdot \mathrm{Lip}(h_\delta) \cdot  \E_{\{\epsilon_i\}_{1}^n} \sup_{f \in \mathcal{F}_k}~ \frac{1}{n} \sum_{i=1}^n \epsilon_i  f(z_i) + \text{(iii)} \\
		& =  \tfrac{2(1+\delta)}{\delta}  \E_{\{\epsilon_i\}_{1}^n} \sup_{\|\bA\|_\star \leq 2^k} \big\langle  \frac{1}{n}\sum_{i=1}^n \epsilon_i \tfrac{\phi_{x_i}}{y_i} \otimes \tfrac{\phi_{x_i}}{y_i}, \bA \big\rangle + \text{(iii)} \\
		& \leq \tfrac{2(1+\delta)}{\delta}  2^k  \underbrace{\E_{\{\epsilon_i\}_{1}^n}\big\| \frac{1}{n}\sum_{i=1}^n \epsilon_i \tfrac{\phi_{x_i}}{y_i} \otimes \tfrac{\phi_{x_i}}{y_i}  \big\|_{\rm op}}_{\text{(iv)}}  + \text{(iii)}
	\end{align*}
	where the last step follows from the duality between the nuclear norm and operator norm. Before getting into the deviation term (iii) (originated by Prop.~\ref{prop:symmetrization}, formally upper bounded in \eqref{eqn:dev}), first recall $2^k \leq 2 (\widehat{\mathsf{Opt}}_n \vee 1)$, we know
	\begin{align*}
		\text{(ii)} \leq \tfrac{4(1+\delta)}{\delta} (\widehat{\mathsf{Opt}}_n \vee 1) \cdot \text{(iv)} + \text{(iii)} \;. \numberthis \label{eqn:lhs-end}
	\end{align*}
	Similarly, by Prop.~\ref{prop:symmetrization}, the deviation term (iii) can be bounded by 
	$6 (\widehat{\mathsf{Opt}}_n \vee 1) \cdot \sup_{x,y}\| \tfrac{\phi_x}{y} \|_{\Hil}^2 \cdot \sqrt{\tfrac{k_0+t}{n}} $ with $k_0 =  \text{c}_{\omega} \log\log(n)$.

	To bound the expected operator norm for the above random matrix, namely term (iv), we rely on matrix Bernstein's inequality plus a truncation technique. Observe that 
	\begin{align*}
		\E_{\epsilon}\big[\epsilon \tfrac{\phi_{x}}{y} \otimes \tfrac{\phi_{x}}{y} \big] = 0, ~\text{and} ~ \| \epsilon \tfrac{\phi_{x}}{y} \otimes \tfrac{\phi_{x}}{y}   \|_{\rm op} \leq \sup_{x,y}\| \tfrac{\phi_x}{y} \|_{\Hil}^2 ~\text{a.s.} \;,
	\end{align*} 
	and that
	\begin{align*}
		\left\| \sum_{i=1}^n \E_{\epsilon} \left[ \big( \epsilon_i \tfrac{\phi_{x_i}}{y_i} \otimes \tfrac{\phi_{x_i}}{y_i} \big)^2 \right] \right\|_{\rm op} \leq \big( \sup_{x,y}\| \tfrac{\phi_x}{y} \|_{\Hil}^2 \big)^2 \cdot n \;.
	\end{align*}
	Naively applying the matrix Bernstein inequality, one would expect the term (iv) to behave like
	$
		\sup_{x,y}\| \tfrac{\phi_x}{y} \|_{\Hil}^2 \cdot \big( \sqrt{\tfrac{t}{n}} \vee \tfrac{t}{n} \big)
	$
	with probability $1 - \text{dim}(\phi_x) \cdot \exp(-t)$ on $\{ \epsilon_i \}_{1}^n$. This is educative yet wrong, since $\text{dim}(\phi)$ is infinity. To make things rigorous, we rely on a truncation technique to look at a finite-dimensional version $\phi_x^{\leq m}$ truncated at a level $m = \text{poly}(n)$ to apply matrix Bernstein, and then estimate the remaining contribution from $\phi_x^{> m}$ by the eigenvalue decay in Assumption \textsf{[S1]}. With details given in Prop.~\ref{prop:infinite-matrix-bernstein}, we derive that
	\begin{align*}
		\text{(iv)} \leq \text{C}_{\tau} \cdot \sup_{x,y}\| \tfrac{\phi_x}{y} \|_{\Hil}^2 \cdot  \big( \sqrt{\tfrac{\log(n)}{n}} \vee \tfrac{\log(n)}{n} \big) \;. \numberthis \label{eqn:matrix-bernstein}
	\end{align*}

	The final piece of the puzzle lies in the term $\sup_{(x,y) \in \mathrm{dom}(\mathcal{P})}\| \tfrac{\phi_x}{y} \|_{\Hil}^2$, which appears in both the main term (iv) and deviation term (iii). It is not true that a.s. for all $x, y$, the above term is bounded. To resolve this issue, we rely on a conditional quantile technique. Introduce the conditional quantile function $Q_{\by^2|\bx = x}(\cdot): [0,1] \rightarrow \Reals_+$ for the conditional random variable $\by^2|\bx = x$.
	Let's only look at data $(x_i, y_i)$'s lying in the region
	\begin{align*}
		\Omega := \{ (x, y)  ~|~ y^2 > Q_{\by^2|\bx = x}(1-\eta) \}, 
	\end{align*}
	and denote $\mathcal{P}|_{\Omega}$ as the conditional distribution of data $(\bx, \by)$ conditioning on the region $\Omega$. 
	Claim that for any $\widehat{\mathsf{PI}}(\bx, \delta)$
	\begin{align}
		\Pr_{(\bx, \by)\sim \mathcal{P}}\big[ \by \notin \widehat{\mathsf{PI}}(\bx, \delta) \big] \leq \Pr_{(\bx, \by)\sim \mathcal{P}|_{\Omega}}\big[ \by \notin \widehat{\mathsf{PI}}(\bx, \delta) \big] \;. \label{eqn:conditional-quantile-trick}
	\end{align}
	This is based on two facts. First, 
	\begin{align*}
		&\Pr\big[ \by^2 > (1+\delta) \widehat{\mathsf{v}}(x) ~|~ \bx = x \big] \\
		&\leq \tfrac{ \Pr\big[\by^2 >  (1+\delta) \widehat{\mathsf{v}}(x) \vee Q_{\by^2|\bx = x}(1-\eta)  ~|~ \bx = x \big] }{ \Pr\big[\by^2 > Q_{\by^2|\bx = x}(1-\eta)  ~|~ \bx = x \big] } \\
		&= \Pr\big[ \by^2 > (1+\delta) \widehat{\mathsf{v}}(x) ~|~ \bx = x,~ (\bx, \by) \in \Omega \big]  \;. \numberthis \label{eqn:marginalize}
	\end{align*}
	regardless of the ordering of $Q_{\by^2|\bx = x}(1-\eta)$ and $(1+\delta) \widehat{\mathsf{v}}(x)$.
	Second, conditioning on $\Omega$ does not change the marginal distribution of $\bx$ due to the quantile construction, namely
	$\mathcal{P}_{\bx} |_{\Omega} \equiv \mathcal{P}_{\bx}$. 
	Marginalizing \eqref{eqn:marginalize} over $x$ proves the above claim. 

	The inequality (\ref{eqn:conditional-quantile-trick}) makes the analyses upper bounding $\text{\small (LHS)}$ from \eqref{eqn:lhs-start}-(\ref{eqn:lhs-end}) applicable, with the changes: (a) $\mathcal{P}|_{\Omega}$ replacing $\mathcal{P}$, and (b) $\widehat{\E}$ denoting average over data points inside $\Omega$ rather than the whole dataset. With the conditioning on $\Omega$, Assumption \textsf{[S2]} implies
	$
	Q_{\by^2|\bx = x}(1-\eta) \geq \xi \cdot K(x, x),
	$
	and thus
	\begin{align*}
		\sup_{(x,y) \in \mathrm{dom}(\mathcal{P}|_{\Omega})}\| \tfrac{\phi_x}{y} \|_{\Hil}^2 \leq \frac{K(x, x)}{Q_{\by^2|\bx = x}(1-\eta) } \leq \xi^{-1} \;. \numberthis \label{eqn:last-piece}
	\end{align*}
	Now, we only need to estimate the effective sample size inside $\Omega$ to complete the analyses. By the quantile construction, 
	$\Pr_{(\bx, \by) \sim \mathcal{P}}\big[ (\bx, \by) \in \Omega \big] = \eta$, a simple Bernstein's inequality asserts that 
	\begin{align*}
		|\{ i~:~ (x_i, y_i) \in \Omega \}| > \tfrac{\eta}{2} \cdot n
	\end{align*}
	with probability at least $1 - \exp(- \text{c}_{\eta} \cdot n)$ on $\{ z_i\}_{1}^n$.

	Finally, plug~\eqref{eqn:last-piece} into upper bounds on terms (iii) and (iv), with $\tfrac{\eta}{2} \cdot n$ replacing $n$ in \eqref{eqn:matrix-bernstein} and \eqref{eqn:lhs-end}, we have proved that
	\begin{align}
		\text{\small (LHS)} &\leq \delta^{-1}  (\widehat{\mathsf{Opt}}_n \vee 1)   \sqrt{\tfrac{ \text{C}_{\tau, \xi, \eta} \log (n)}{n}} \quad \text{(main term)} \\
		& \quad + (\widehat{\mathsf{Opt}}_n \vee 1) \sqrt{\tfrac{ \text{C}_{\xi, \eta, \omega} (\log\log(n) +  t)}{n}} \quad \text{(deviation term)} \label{eqn:dev}
	\end{align}
	with probability at least $1 - \exp(- \text{c}_{\eta} \cdot n) - 2\exp(-t)$ on $\{ z_i\}_{1}^n $.

\section{Numerical Studies}
We now study the numerical performance of our procedure.

\subsection{Empirical Example: Comparison to Conformal Prediction}
\label{sec:comparision-conformal}
This section compares our SDP methods to the conformal prediction methods, measuring the coverage and statistical efficiency of the prediction bands. We compare five methods on two simulated data sets, including (a) standard prediction band using simple linear regression (SLR), without accounting for heteroscedasticity, as a baseline; (b) SDP prediction bands proposed in this paper, with two specifications of the kernels, denoted as SDP1 and SDP2; (c) conformal prediction bands, including full conformal prediction (CF) and split conformal prediction (SplitCF), see \cite[Algorithms 1 and 2 respectively]{lei2018DistributionFreePredictive}. For each method, we report the coverage probability, efficiency statistics (including median length and average length), and finally, the mean squared error (MSE) of the estimated conditional mean function. 

\begin{figure*}[tbhp]
	\centering
    \begin{subfigure}{.5\textwidth}
        \centering
        \includegraphics[width=\linewidth]{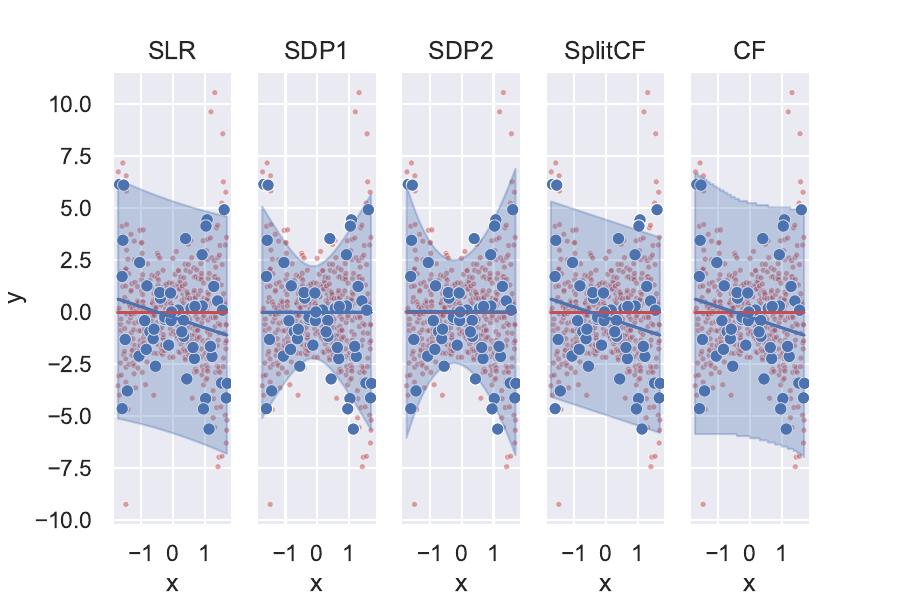}
		\caption{}
        \label{fig:comparison_Gauss}
    \end{subfigure}%
    \begin{subfigure}{0.5\textwidth}
        \centering
        \includegraphics[width=\linewidth]{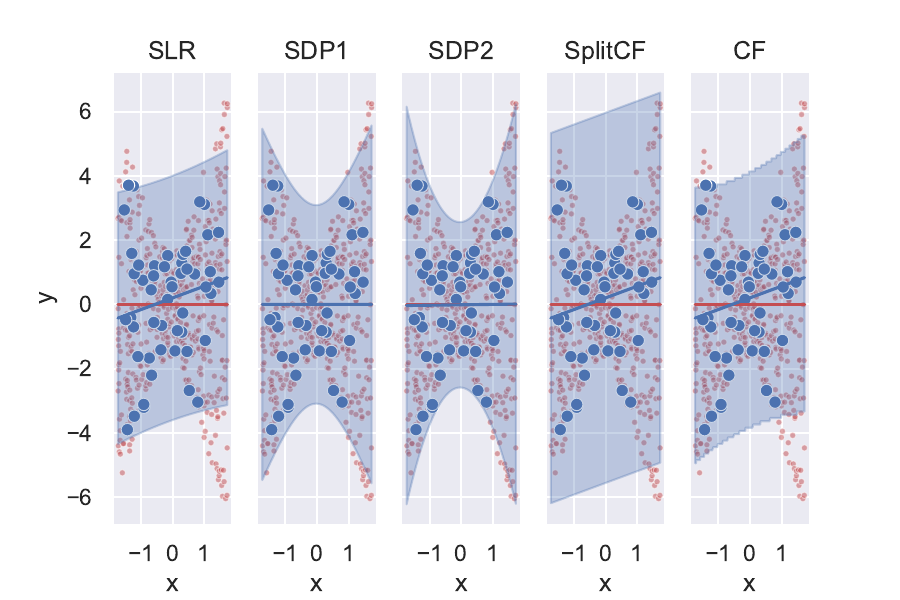}
		\caption{}
        \label{fig:comparison_Unif}
    \end{subfigure}
	\caption{\scriptsize From left to right: SLR, SDP1, SDP2, split conformal (SplitCF), and full conformal (CF). Same style as in Fig.~\ref{fig:intro_example1}-\ref{fig:intro_example2}. Statistics are summarized in Table~\ref{table:comparison}.}
\end{figure*}

We first explain the two simulated data sets. Here the $\bx \sim \text{Unif}([-\sqrt{3}, \sqrt{3}])$, and the conditional distribution $\by| \bx = x \sim \epsilon \cdot \sqrt{\mathsf{v}(x)}$ where the independent error $\epsilon$ is either drawn from a standard normal $\text{N}(0,1)$ (Fig.~\ref{fig:comparison_Gauss}) or a uniform distribution $\text{Unif}([-\sqrt{3}, \sqrt{3}])$ (Fig.~\ref{fig:comparison_Unif}). The conditional mean function $\mathsf{m}(x)=0$. The heteroscedastic variance function scales as $\mathsf{v}(x) = 1 + x + 4x^2$, depending on $x$. For each simulated data set, we generate $\{(x_i, y_i)\}_{i=1}^n$ i.i.d. from the above data generating process (DGP), with $n = 50$ as the training data, marked by \textcolor{blue}{Blue dots}. The test data set are drawn from the same DGP, with $N=500$ marked by \textcolor{red}{Red dots}. For the SDP1/SDP2, and SplitCF, an independent calibration data set with $n=50$ is used. The calibration set is used to choose $\delta$ in SDPs as in Algorithm~\ref{alg:calibration}, and the homoscedestic conformal bandwidth as in SplitCF. We compare five methods that construct the prediction bands, illustrated by the \textcolor{blue}{Blue band}, on the Gaussian error data set in Fig.~\ref{fig:comparison_Gauss}, and Uniform error data set in Fig.~\ref{fig:comparison_Unif}. For all methods, the desired coverage is set at $1-\alpha = 95\%$. For the SDPs in \eqref{eqn:SDP-simultaneous} $\gamma = 10$. Table~\ref{table:comparison} summarizes the coverage, efficiency, and estimation error.

\begin{table}[tbhp]
\centering
\caption{}
\label{table:comparison}
\begin{tabular}{lrrrr}
	 & Coverage & Median Len & Average Len & MSE \\
\midrule
\multicolumn{4}{l}{Fig.~\ref{fig:comparison_Gauss}: $\mathsf{m}(x) = 0$, quadratic $\mathsf{v}(x)$, $\by|\bx \sim$ Gaussian} \\
\midrule
SLR  & 97.40\% & 11.2272 & 11.2537 & 0.5554 \\
SDP1 & 92.80\% & 6.5172 & 6.9019 & 0.0002\\
SDP2 & 94.20\% & 7.0025 & 7.6900 & 0.0272 \\
SplitConformal & 96.00\% & 9.3960 & 9.3960 & 0.5554 \\
Conformal & 97.40\% & 11.5152 & 11.5911 & 0.5554 \\
\midrule
\multicolumn{4}{l}{Fig.~\ref{fig:comparison_Unif}: $\mathsf{m}(x) = 0$, quadratic $\mathsf{v}(x)$, $\by|\bx \sim$ Uniform} \\
\midrule
SLR  & 89.60\% & 7.6882 & 7.7077 & 0.4405 \\
SDP1 & 95.40\% & 7.9161 & 8.1540 & 0.0000 \\
SDP2 & 96.60\% & 7.3064 & 7.8175 & 0.0183 \\
SplitConformal & 97.80\% & 11.5199 & 11.5199 & 0.4405 \\
Conformal & 92.80\% & 8.0808 & 8.1651 & 0.4405 \\
\bottomrule
\end{tabular}
\end{table}

Nearly all methods achieve $95\%$ desired coverage, with the only exception of SLR. The focus will be on comparing efficiencies, namely, which method estimates a smaller, truly heteroscedastic band in achieving the desired coverage.
As seen visually in Fig.~\ref{fig:comparison_Gauss}-\ref{fig:comparison_Unif} and numerically in Table~\ref{table:comparison}, the two conformal methods, SplitCF and CF, estimates a conservative, wide prediction band that is almost homoscedastic. In contrast, both SDP approaches estimate desirable heteroscedastic bands that are on average much shorter, with the closest ($94.20\%$ and $95.40\%$) to the desired $95\%$ coverage. Finally, we would like to remark that all four methods SLR, SDP1, SDP2 and SplitCF are efficient to compute. Yet, the full conformal method CF involves discretizing input space, which is computationally intensive. Our empirical results show that the two conformal methods can be unnecessarily conservative and form bands of nearly constant width across $x$ \citep{romano2019ConformalizedQuantile}.

%
%
%
%
%
%
%
%

%
%
%
%
%
%
%
%

\subsection{Real Data Example: Fama-French Factors}
\label{sec:Real-Data-Fama-French}


In this section, we apply our method of constructing the prediction band to the celebrated three-factor dataset created by \cite{fama1993CommonRisk}. We choose this dataset for three reasons: (a) financial data are known to suffer severe heteroscedasticity, (b) the factors are believed to be different sources explaining returns of diversified portfolios, thus when conditioned on one factor, the other factors should have large, heteroscedastic conditional variability, and (c) the factors---Market, Size, and Value--- correspond nicely to our common sense about the financial market for exploratory data analysis. 

Let us first explain the data in plain language. The dataset consists of yearly and monthly observations of four variables from July 1926 to December 2020. The four variables are (a) Risk-free return rate (\textsf{\small RF}), the one-month Treasury bill rate (i.e., interest rate), (b) Market factor (\textsf{\small MKT}), the excess return on the market (i.e., market return minus interest rate), (c) Size factor (\textsf{\small SMB}, Small Minus Big), the average difference in returns between small and big portfolios according to the market capitalization, and (d) Value factor (\textsf{\small HML}, High Minus Low), the difference in returns between value and growth portfolios. We design two experiments, one focusing on the prediction coverage and bandwidth, and the other on exploring the role of the tuning parameter $\gamma$ in trading off mean and variance. 

The first experiment aims to access the prediction coverage in the SDP~\eqref{eqn:SDP-simultaneous}, using \textsf{\small MKT} (as $x$) to predict other variables (as $y$): \textsf{\small RF} and two other factors \textsf{\small SMB} and \textsf{\small HML}. Here we use yearly data ($n = 94$ from 1927-2020, shown as \textcolor{blue}{Blue dots}) to construct the prediction bands, each illustrated in Fig.~\ref{fig:RF-yearly}-\ref{fig:HML-yearly}. As for the test data, we use the standardized monthly data ($N = 1134$, normalized to zero mean and unit standard deviation, shown as \textcolor{red}{Red dots}) as a surrogate for test $(x, y)$ pairs. Namely, we match $12$ test data to each training data. We verified that after standardization, the histograms of yearly and monthly data match nicely for all four variables. For each type of response variable, we run two SDPs with different kernels. The summary statistics about the coverage probability, median, and mean bandwidth are given in Table~\ref{table:fama-french}. 


\begin{figure*}[tbhp]
	\centering
	    \begin{subfigure}{0.33\textwidth}
	        \centering
	        \includegraphics[width=\linewidth]{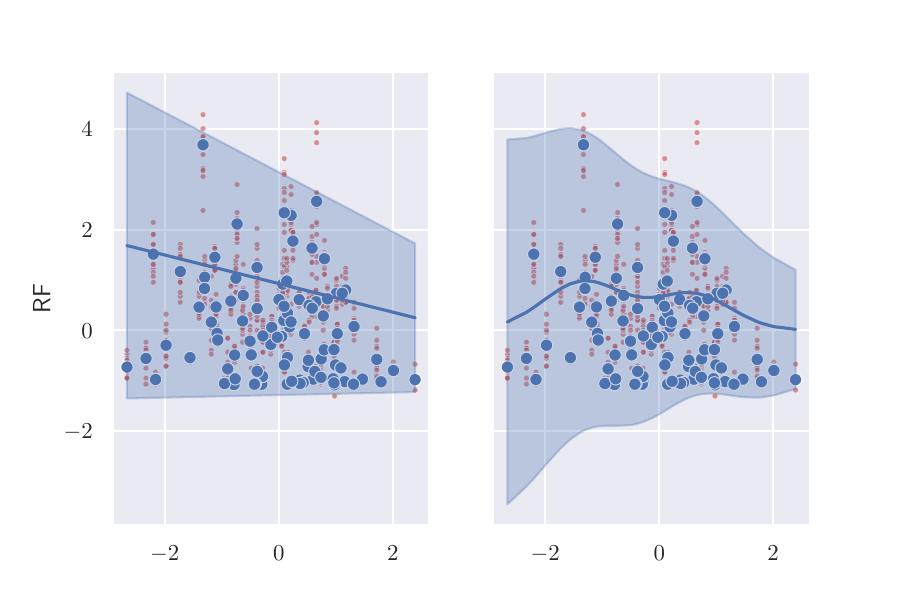}
	        \caption{RF}
	        \label{fig:RF-yearly}
	    \end{subfigure}
	    \begin{subfigure}{0.33\textwidth}
	        \centering
	        \includegraphics[width=\linewidth]{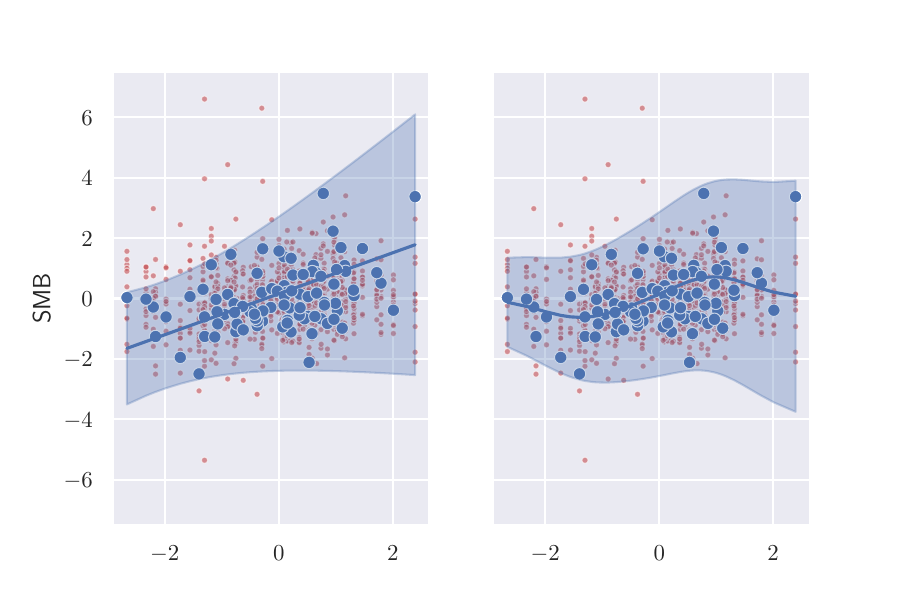}
	        \caption{SMB}
	        \label{fig:SMB-yearly}
	    \end{subfigure}%
	    \begin{subfigure}{0.33\textwidth}
	        \centering
	        \includegraphics[width=\linewidth]{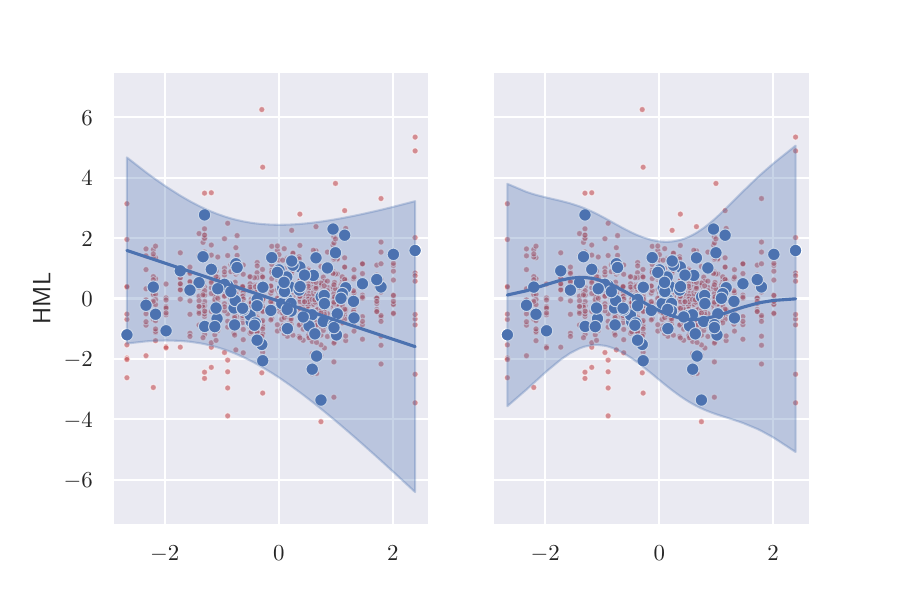}
	        \caption{HML}
	        \label{fig:HML-yearly}
	    \end{subfigure}
		\caption{\scriptsize From left to right: response variable $y$ corresponds to \textsf{\small RF}, \textsf{\small SMB} and \textsf{\small HML}, with $x$ being \textsf{\small MKT}. \textcolor{blue}{Blue dots} denote $n=94$ training data $\{(x_i, y_i)\}_{i=1}^n$, \textcolor{blue}{Blue line} denotes the estimated conditional mean $\widehat{\mathsf{m}}(x)$, and \textcolor{blue}{Blue band} denotes the estimated prediction band $\widehat{\mathsf{PI}}(x)$. \textcolor{red}{Red dots} represent $N=1134$ test data points.}
\end{figure*}

\begin{table}[tbhp]
\centering
\caption{Real data: Fama-French}
\label{table:fama-french}
\begin{tabular}{lcrrr}
 & Kernel & Coverage & Median Len & Average Len \\
 \midrule
 \textsf{RF} & lin $\mathsf{m}(x)$, quad $\mathsf{v}(x)$ & \textbf{98.68\%} & \textbf{4.3616} & \textbf{4.4358} \\
 \textsf{RF} & rbf $\mathsf{m}(x)$, quad $\mathsf{v}(x)$ & 98.59\% & 4.5693 & 4.6847 \\
\midrule
\textsf{SMB} & lin $\mathsf{m}(x)$, quad $\mathsf{v}(x)$ & 95.77\% & \textbf{5.2560} & \textbf{5.2798}  \\
\textsf{SMB} & rbf $\mathsf{m}(x)$, quad $\mathsf{v}(x)$ & \textbf{97.53\%} & 5.5407 & 5.4290 \\
\midrule
\textsf{HML} & lin $\mathsf{m}(x)$, quad $\mathsf{v}(x)$ & 96.56\% & 5.2822 & 5.5556 \\
\textsf{HML} & rbf $\mathsf{m}(x)$, quad $\mathsf{v}(x)$ &\textbf{97.27\%} & \textbf{4.9180} & \textbf{5.3640} \\
\bottomrule
\end{tabular}
\end{table}

We note a few observations regarding the empirical results. First, all models achieve desirable coverage (all $>95\%$). Second, controlling for \textsf{\small MKT}, all other factors have significant heteroscedastic error left unexplained. For the \textsf{\small RF}, a high \textsf{\small MKT} return implies a low expected \textsf{\small RF} interest, and more importantly, a small variability, compared to the low \textsf{\small MKT} return case. For the size factor \textsf{\small SMB}, the conditional variability is much larger when the \textsf{\small MKT} is high vs. low, so does the conditional expectation. The conditional variability in \textsf{\small SMB} is roughly minimized when the market is significantly below average. While for the value factor \textsf{\small HML}, conditional variability is minimized when the market is slightly below its average. 

The second experiment aims to verify the mean and variance trade-offs by tuning the parameter $\gamma$, discussed in Sec.~\ref{sec:sos-mni}. Here we use the monthly return data, and for each sub-experiment, we split the data into (train, valid, test) parts. We train models with different $\gamma = 0.1, 1, 10$ on the training data, then valid their performances on the validation data. We finally evaluate the performance using the test data with the cross-validated optimal $\gamma$ (based on the validation data). A nice feature about this experiment is that, one can visualize how the SDPs trade a complex/large conditional variance $\mathsf{v}(x)$ for a simple/small conditional mean $\mathsf{m}(x)$ in explaining $\by|\bx=x$ as $\gamma$ increases, illustrated by Fig.~\ref{fig:gamma}.

%
%

%

%

%

\section{Summary}
The current paper progresses to resolve the uncertainty quantification dilemma faced by modern machine learning models. There are two innovative viewpoints we are taking. First, rather than relying on idealized parametric distributional assumptions on error $\by-\mathsf{m}(\bx)$, we make minimal assumptions.  Both the conditional mean and variance functions are modeled nonparametrically and can universally approximate all functions. It is worth noting that such flexibility does not hinder computational feasibility due to the sum-of-squares and convex relaxations. The computational complexity and statistical guarantee scale favorably with high-dimensional covariates $x$. Second, rather than modeling the conditional mean only and giving up the variance (Frequentist justification, the conditional mean is assumed inside an RKHS, see \citep{caponnetto2007optimal, liang2020JustInterpolate, liang2020MultipleDescent}), or modeling the conditional variance function only (Bayesian justification of kriging/Gaussian processes regression, the covariance function is specified by a kernel, see \citep{handcock1993bayesian,stein2005space, stein2012interpolation}) for the variability in data, we model both the mean and the variance and prove strong, non-asymptotic Frequentist coverage guarantees. Such a modeling advantage enables the uncertainty quantification with or without any black-box predictive model, whether accurate or not.

To conclude, our Theorem~\ref{thm:coverage} established a strong, non-asymptotic coverage guarantee in the language of Neyman, yet with two distinct new features. First, the coverage probability can go to $1$ with a fixed confidence parameter $\delta$ as long as the sample size $n$ is large enough. Second, the data-adaptive quantity $\widehat{\mathsf{Opt}}_n$ controls both the average bandwidth and the coverage guarantee of the prediction band $\widehat{\mathsf{PI}}(x)$. A small objective value of the SDP makes the prediction band accurate and narrow simultaneously. Finally, our procedure for constructing prediction bands can be viewed as a novel variance interpolation with confidence and further leverages techniques from semi-definite programming and sum-of-squares optimization. We conducted simulated and real data experiments to validate the prediction interval's numerical performance for uncertainty quantification. A minimal 10-line Python implementation is provided in Listing.~\ref{code} for interested readers.


\begin{figure*}[tbhp]
	\centering
	    \begin{minipage}{0.5\textwidth}
	        \centering
	        \includegraphics[width=\linewidth]{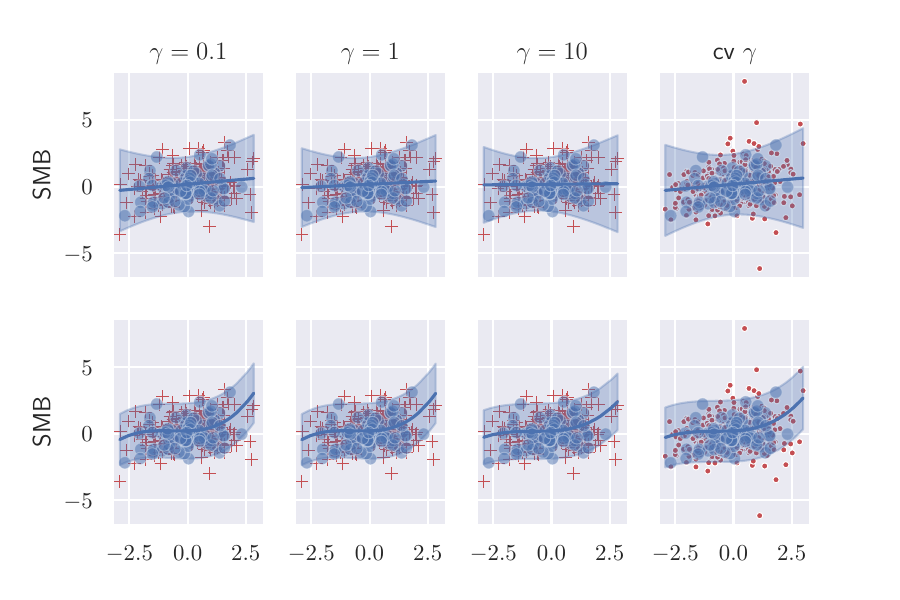}
	    \end{minipage}%
	    \begin{minipage}{0.5\textwidth}
	        \centering
	        \includegraphics[width=\linewidth]{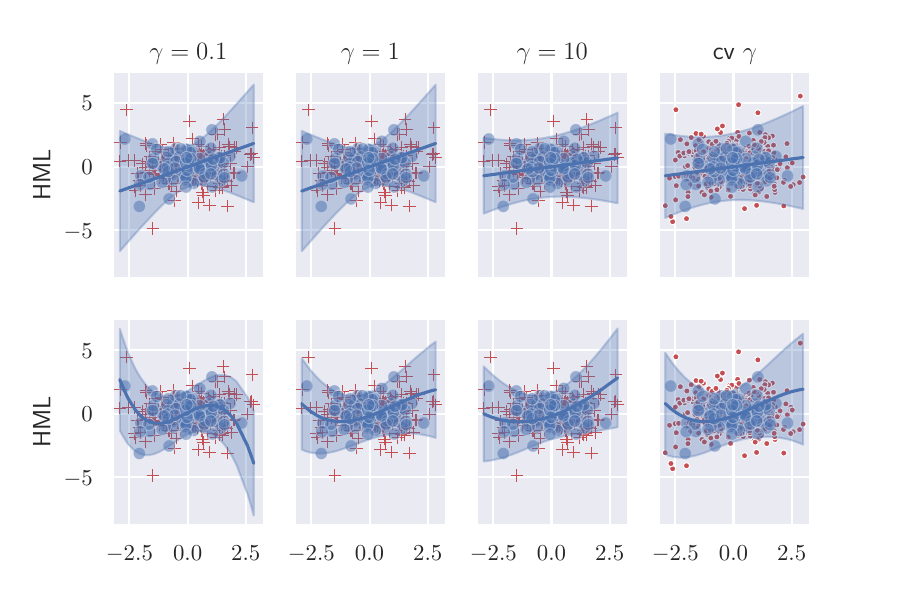}
	    \end{minipage}
		\caption{\scriptsize Cross-validate $\gamma$ experiment. On the left is \textsf{\small SMB} as the response and on the right is \textsf{\small HML}. For each response variable, we run two sub-experiments: the top row corresponds to a linear $\mathsf{m}(x)$ and a quadratic $\mathsf{v}(x)$, and the bottom row corresponds to a degree-$3$ polynomial $\mathsf{m}(x)$ and a quadratic $\mathsf{v}(x)$. Each sub-figure corresponds to a specific $\gamma$, noted in its title. The cv $\gamma$ denotes the cross-validated optimal $\gamma$ using the validation dataset. Here the (train, valid, test) dataset has size proportional to 1:3:9, denoted by \textcolor{blue}{Blue dots}, \textcolor{red}{Red pluses}, and \textcolor{red}{Red dots} respectively.}
		\label{fig:gamma}
\end{figure*}

\section*{Acknowledgement}

Liang acknowledges the generous support from the NSF Career award (DMS-2042473)
and the William S. Fishman fellowship. Liang thanks Michael Stein, Ruey Tsay, and Vladimir Vovk for constructive comments. Liang thanks the editor, associate editor, and three anonymous referees for the valuable suggestions that improve the paper.

\bibliography{ref}
\bibliographystyle{unsrtnat}


\newpage

\appendix
\section{Appendix}

\subsection{Remaining Propositions}
In this section, we collect the remaining propositions.
\begin{proposition}[Representation]
	\label{prop:representation}
	The kernelized version of the SDP as in \eqref{eqn:SDP-simultaneous} is equivalent to the following infinite-dimensional SDP
	\begin{align*}
	\min_{\beta \in \Hil^{\mathsf{m}}, ~\bA: \Hil^{\mathsf{v}} \rightarrow \Hil^{\mathsf{v}}} \quad & \gamma \cdot \| \beta \|_{\Hil^{\mathsf{m}}}^2 +  \| \bA \|_{\star} \\
	\text{s.t.}\quad & \langle \phi_{x_i}^{\mathsf{v}} , \bA \phi_{x_i}^{\mathsf{v}} \rangle_{\Hil^{\mathsf{v}}} \geq \big( y_i - \langle  \phi_{x_i}^{\mathsf{m}}, \beta \rangle_{\Hil^{\mathsf{m}}} \big)^2, ~ \forall i \enspace. \\
			& \bA \succeq 0
	\end{align*}
\end{proposition}
\textbf{Proof.} Noticing that the solution to the infinite-dimensional problem must lie in the span of data, namely
$\bA = \sum_{i,j} \bB_{ij} \phi_{x_i}^{\mathsf{v}} \otimes \phi_{x_j}^{\mathsf{v}}$ with some PSD $\bB \in \mathbb{S}^{n\times n}$, and $\beta = \sum_{i} \alpha_i \phi_{x_i}^{\mathsf{m}}$ with some $\alpha \in \Reals^n$. With the above representation, plug in the infinite-dimensional SDP and recall $\tr(\bA) = \| \bA \|_{\star}$, we can derive \eqref{eqn:SDP-simultaneous}.
When $\beta$ is not a decision variable, this representation theorem applies to \eqref{eqn:SDP-general}. \qed

\begin{proposition}[Symmetrization]
	\label{prop:symmetrization}
	Let $\mathcal{F}$ be a class of functions $f:\mathcal{Z} \rightarrow \Reals$, with $\sup_{x\in \mathcal{Z}} |f(z)| \leq M$. Then with probability at least $1- 2\exp(-t)$ on $\{z_i\}_{i=1}^n$ i.i.d. drawn from a distribution, we have
	\begin{align*}
		\sup_{f \in \mathcal{F}} \big| \E[f(\bz)]-\widehat{\E}[f(\bz)] \big| \leq 2 \cdot \E_{\epsilon} \sup_{f\in \mathcal{F}} \frac{1}{n} \sum_{i=1}^n \epsilon_i f(z_i) + 3M \sqrt{\frac{t}{2n}} \enspace.
	\end{align*}
\end{proposition}
\textbf{Proof.} First, with McDiarmid's inequality, we know w.h.p.
\begin{align*}
	\sup_{f \in \mathcal{F}} \big| \E[f(\bz)]-\widehat{\E}[f(\bz)] \big| \leq  \E_{\{z_i\}_{1}^{n}} \sup_{f \in \mathcal{F}} \big| \E[f(\bz)]-\widehat{\E}[f(\bz)] \big| + M \sqrt{\frac{t}{2n}} \;.
\end{align*}
Apply Gin\'{e}-Zinn symmetrization to the first term on the RHS, then apply McDiarmid's inequality again, we can establish the claim. See \cite{liang2020JustInterpolate} for details. \qed

\begin{proposition}[Objective value estimate]
	\label{prop:growth-opt}
	Under \textsf{[S3]}, the following holds with probability at least $1- n^{-10}$,
	\begin{align*}
		 \widehat{\mathsf{Opt}}_n \leq  [\log(n)]^{ \text{c}_{\omega}} \;.
	\end{align*}
\end{proposition}
\textbf{Proof.}
	Apply union bound on the tails given by \textsf{[S3]}, with the choice $t_0 = [\log(n)]^{ \text{c}_{\omega}}$ we know
	$$
	\Pr\big[ y_i^2 \leq t_0 \cdot K(x_i, x_i), ~\forall~ 1\leq i \leq n \big] \geq 1-  n \cdot \exp(-\text{C} t_0^{\omega}) \geq 1 - n^{-10}.
	$$
	In addition, it is easy to verify that there exist a finite rank $r$ operator such that
	$
	2 \langle \phi_{x_i}, \bI_r \phi_{x_i} \rangle_{\Hil} \geq \cdot \langle \phi_{x_i},  \phi_{x_i} \rangle_{\Hil}.
	$
	In view of Prop.~\ref{prop:representation}, the above certifies that $\bA :=  2t_0 \cdot \bI_r$ lies in the feasibility set $\langle \phi_{x_i}, \bA \phi_{x_i} \rangle_{\Hil} = t_0 \cdot K(x_i, x_i) \geq y_i^2$, which implies $2r t_0$ being an upper bound on 
$\widehat{\mathsf{Opt}}_n$. \qed

\begin{proposition}[Operator-norm estimate]
	\label{prop:infinite-matrix-bernstein}
	Under \textsf{[S1]}, for any $\{ x_i \}_{i=1}^n$, the following holds
	\begin{align*}
		\E_{\{\epsilon_i\}_{1}^n}\big\| \frac{1}{n}\sum_{i=1}^n \epsilon_i \phi_{x_i}\otimes \phi_{x_i} \big\|_{\rm op} \leq \text{C}_{\tau} \cdot \sup_{x}\| \phi_x \|_{\Hil}^2 \cdot  \big( \sqrt{\tfrac{\log(n)}{n}} \vee \tfrac{\log(n)}{n} \big) \;. 
	\end{align*}
\end{proposition}
\textbf{Proof.} Recall \textsf{[S1]}, due to the Mercer's theorem, one can represent $\phi_{x}$ as an infinite-dimensional vector, with each coordinate of $\phi^{j}_x$ corresponding to the eigenfunction of the integral operator, with $j=1,\ldots, \infty$ and $\lambda_j \leq \text{C} j^{-\tau}$. To bound the operator norm, recall the Rayleigh quotient form, for any $h \in \Hil$ with $\|h\|_{\Hil}^2 = 1$
\begin{align}
	\left| \big\langle h, \big(\tfrac{1}{n}\sum_{i} \epsilon_i \phi_{x_i}\otimes \phi_{x_i} \big)  h \big\rangle_{\Hil} \right| = \left| \tfrac{1}{n}\sum_{i} \epsilon_i \langle \phi_{x_i}, h \rangle_{\Hil}^2  \right|  \;. \label{eqn:rayleigh-quotient}
\end{align}
Note $\langle \phi_{x}, h \rangle_{\Hil} = \langle \phi_{x}^{\leq m}, h^{\leq m} \rangle_{\Hil} + \langle \phi_{x}^{>m}, h^{> m} \rangle_{\Hil}$ where the superscript indicates a truncation on the coordinates of $\phi_{x}$. We know
$$
\langle \phi_{x_i}, h \rangle_{\Hil}^2 =  \langle \phi_{x_i}^{\leq m}, h^{\leq m} \rangle_{\Hil}^2 +  \langle \phi_{x_i}^{>m}, h^{> m} \rangle_{\Hil}^2 + 2 \langle \phi_{x_i}^{\leq m}, h^{\leq m} \rangle_{\Hil}  \langle \phi_{x_i}^{>m}, h^{> m} \rangle_{\Hil} \;.
$$
Therefore LHS in \eqref{eqn:rayleigh-quotient} can be upper bounded by 
\begin{align*}
	 \big\| \frac{1}{n}\sum_{i=1}^n \epsilon_i \phi_{x_i}^{\leq m}\otimes \phi_{x_i}^{\leq m} \big\|_{\rm op} +  \sup_{i} \|  \phi_{x_i}^{>m} \|_{\Hil}^2 + \text{C} \cdot \sup_{i} \|  \phi_{x_i}^{>m} \|_{\Hil} \;,
\end{align*}
For some absolute constant $\text{C}>0$.
For the first term, now we can apply the matrix Bernstein inequality. With probability $1-2 \exp(-t)$ the following upper bound on the first term holds
\begin{align*}
	\sup_{x} \| \phi_x \|_{\Hil}^2 \cdot \big( \sqrt{\tfrac{\log(m) + t}{n}} \vee \tfrac{\log(m) + t}{n} \big) \;.
\end{align*}
For the second term, recall the eigenvalue decay $\lambda_j \leq \text{C} j^{-\tau}$ with $\tau>1$, we know it is upper bounded by
$\text{C} m^{-(\tau-1)}$. Choosing $\log(m) = \text{C}_{\tau} \log(n)$ with a constant large enough, we know the second and third terms are dominated by the first term. By integrating the tail bound to obtain a bound on the expected value, we complete the proof. \qed

\subsection{Proof of the Calibration Lemma}
\label{sec:calibration-proof}
\begin{proof}[Proof of Lemma~\ref{lem:calibration}]
	Define the Bernoulli random variables $z_j(\delta) := \mathbbm{1} [y_j' \notin \widehat{\mathsf{PI}}(x_j', \delta)]$, then for any fixed $\delta \in [-1, \Delta]$, we know by Hoeffding's inequality that
	\begin{align}
		\big| \Pr_{(\bx, \by)\sim \mathcal{P}}\big[ \by \notin \widehat{\mathsf{PI}}(\bx, \delta) \big] - \tfrac{1}{m} \sum_{j=1}^m \mathbbm{1} [y_j' \notin \widehat{\mathsf{PI}}(x_j', \delta)] \big| \leq \sqrt{\tfrac{t}{2m}} \;
	\end{align}
	with probability $1 - 2\exp(-t)$ on $\{ x'_j, y'_j\}_{j=1}^m$. Therefore, if we can identify a $\delta$ such that
	\begin{align}
		\tfrac{1}{m} \sum_{j=1}^m \mathbbm{1} [y_j' \notin \widehat{\mathsf{PI}}(x_j', \delta)] \leq \tfrac{3}{4}\alpha \;,
	\end{align}
	and for some later specified choice of $t$ that
	\begin{align}
		\sqrt{\tfrac{t}{2m}} \leq \tfrac{1}{4}\alpha \;,
	\end{align}
	the proof will complete. Observe that both $\Pr_{(\bx, \by)\sim \mathcal{P}}\big[ \by \notin \widehat{\mathsf{PI}}(\bx, \delta) \big]$ and its empirical counterparts $\tfrac{1}{m} \sum_{j=1}^m \mathbbm{1} [y_j' \notin \widehat{\mathsf{PI}}(x_j', \delta)]$ are monotonic in $\delta$.
	We claim that the Algorithm~\ref{alg:calibration} will terminate with at most $\log_2\big(\tfrac{2L(\Delta+1)}{\alpha}\big)$ iterations (namely, disjoint choices of $\delta$ in the while loop). To prove this claim, we note that the algorithm must terminate in the interval $\delta \in [\Delta - \tfrac{\alpha}{2L} , \Delta]$ since we know
	\begin{align}
		\tfrac{1}{m} \sum_{j=1}^m \mathbbm{1} [y_j' \notin \widehat{\mathsf{PI}}(x_j', \Delta - \tfrac{\alpha}{2L})] \leq \Pr_{(\bx, \by)\sim \mathcal{P}}\big[ \by \notin \widehat{\mathsf{PI}}(\bx, \Delta - \tfrac{\alpha}{2L}) \big] +  \sqrt{\tfrac{t}{2m}} \leq L \tfrac{\alpha}{2L} + \tfrac{1}{4}\alpha \leq \tfrac{3}{4}\alpha \;,
	\end{align}
	by the mean value theorem and the upper bound on the Lipschitz constant. With the dyadic search, the algorithm will terminate after at most $\lceil \log_2\big(\tfrac{2L(\Delta+1)}{\alpha}\big) \rceil$ pre-determined dyadic grids of $\delta$'s with the form $G_{\rm dyadic}:= \{ (1 - \tfrac{1}{2^k})\Delta - \tfrac{1}{2^k} ~|~ k = 0, 1, \ldots,  \lceil \log_2\big(\tfrac{2L(\Delta+1)}{\alpha}\big) \rceil \}$. Therefore, take $t = \log \big( \lceil \log_2\big(\tfrac{2L(\Delta+1)}{\alpha}\big) \rceil + 1 \big) + 10\log (m)$ and recall that $m$ is large enough such that
	\begin{align*}
		\sqrt{\tfrac{\log \big( \lceil \log_2\big(\tfrac{2L(\Delta+1)}{\alpha}\big) \rceil + 1 \big) + 10\log (m)}{m}} \leq \tfrac{1}{4}\alpha \;,
	\end{align*}
	then uniformly over the fixed dyadic grid $\delta \in G_{\rm dyadic}$, 
	\begin{align}
		 \sup_{\delta \in G_{\rm dyadic}} ~\Pr_{(\bx, \by)\sim \mathcal{P}}\big[ \by \notin \widehat{\mathsf{PI}}(\bx, \delta) \big] - \tfrac{1}{m} \sum_{j=1}^m \mathbbm{1} [y_j' \notin \widehat{\mathsf{PI}}(x_j', \delta)] \leq \tfrac{1}{4}\alpha
	\end{align}
	with probability at least $1 - 2m^{-10}$. It is easy to see that $\delta^\star(\alpha) \in G_{\rm dyadic}$, and thus on the same event, 
	\begin{align}
		\Pr_{(\bx, \by)\sim \mathcal{P}}\big[ \by \notin \widehat{\mathsf{PI}}(\bx, \delta^\star(\alpha)) \big] \leq \tfrac{1}{m} \sum_{j=1}^m \mathbbm{1} [y_j' \notin \widehat{\mathsf{PI}}(x_j', \delta^\star(\alpha))] + \tfrac{1}{4}\alpha \leq \alpha \;.
	\end{align}
\end{proof}

\subsection{Remaining Experimental Details}


All experiments are conducted using the Python language. The minimal implementation is provided below
{\scriptsize
\begin{lstlisting}[language=Python, caption=Minimal python code, label=code]
import cvxpy as cp

def sdpDual(K1, K2, Y, n, gamma = 1e1):
# K1 kernel for conditional mean, 1st moment
# K2 kernel for conditional variance, 2nd moment   
# Define and solve the CVXPY problem.
    # Create a symmetric matrix variable \hat{B} 
    hB = cp.Variable((n,n), symmetric=True)
    # Create a vector variable \hat{a}
    ha = cp.Variable(n)
	
	# PSD and inequality constraints
    constraints = [hB >> 0] 
    constraints += [
        K2[i,:]@hB@K2[i,:] >= 
        cp.square(Y[i] - K1[i,:]@ha) for i in range(n) 
    ]
    prob = cp.Problem(cp.Minimize(
        gamma*cp.quad_form(ha, K1) + cp.trace(K2@hB)
    ), constraints)

    # Solve the SDP
    prob.solve()
    print("Optimal_Value", prob.value)
    
    return [ha.value, hB.value]
\end{lstlisting}
}

\end{document}